%% file: arxiv.tex
\documentclass[a4paper]{article}

\usepackage[utf8]{inputenc}
\usepackage[T1]{fontenc}

\usepackage[a4paper,margin=1in]{geometry}

\usepackage[]{natbib}
\usepackage[usenames,dvipsnames]{xcolor}

\usepackage[colorlinks = true,
    linkcolor = blue,
    anchorcolor = blue,
    citecolor = blue,
    filecolor = blue,
    urlcolor = black,
    pagebackref]{hyperref}

\input{header.tex}

\title{Online Learning and Solving Infinite Games with an ERM Oracle}

\author{
     Angelos Assos \thanks{MIT CSAIL; \texttt{assos@mit.edu}.}
    \and 
    Idan Attias \thanks{Department of Computer Science, Ben-Gurion University; \texttt{idanatti@post.bgu.ac.il}.} 
    \and
    Yuval Dagan \thanks{MIT CSAIL; \texttt{yuvaldag@gmail.com}.}
    \and
    Constantinos Daskalakis \thanks{MIT CSAIL; \texttt{costis@csail.mit.edu}.}
    \and
    Maxwell Fishelson \thanks{MIT CSAIL; \texttt{maxfish@mit.edu}.}
}

\begin{document}
\maketitle

\begin{abstract}%
While ERM suffices to attain near-optimal generalization error in the stochastic learning setting, this is not known to be the case in the online learning setting, where algorithms for general concept classes rely on computationally inefficient oracles such as the Standard Optimal Algorithm (SOA). In this work, we propose an algorithm for  online binary classification setting that relies solely on ERM oracle calls, and show that it has finite regret in the realizable setting and sublinearly growing regret in the agnostic setting. We bound the regret in terms of the Littlestone and threshold dimensions of the underlying concept class.

We obtain similar results for nonparametric games, where the ERM oracle can be interpreted as a best response oracle, finding the best response of a player to a given history of play of the other players. In this setting, we provide learning algorithms that only rely on best response oracles and converge to approximate-minimax equilibria in  two-player zero-sum games and  approximate coarse correlated equilibria in multi-player general-sum games, as long as the game has a bounded fat-threshold dimension. Our algorithms apply to both binary-valued and real-valued games and can be viewed as providing justification for the wide use of double oracle and multiple oracle algorithms in the practice of solving large games.
\end{abstract}

\section{Introduction}

The advent of Deep Learning has exacerbated the importance of learning models which involve a large number of parameters or are non-parametric. {Non-parametric learning is learning at its fullest generality.  We make no assumption about the structure of our decision space, working with potentially infinite and non-continuous hypothesis classes.}  From a theoretical standpoint, most study of non-parametric learning has focused on the stochastic setting, where one learns a model given independent observations from some distribution. This study has led to important developments in---both frequentist and Bayesian---non-parametric Statistics, including the discovery of notions of complexity of hypotheses classes, such as the celebrated VC and fat-shattering dimensions, which tightly capture the number of observations from a distribution that is necessary to select a hypothesis whose prediction error under the distribution is approximately optimal. In fact, this is achieved via Empirical Risk Minimization (ERM), the simple method, which given a hypothesis class ${\cal H}$, of functions mapping a feature set $\cal X$ to a label set ${\cal Y} \subseteq \mathbb{R}$, 
and a set of observations $\{(x_i,y_i)\}_{i=1}^N$, outputs:\footnote{To be formal, if the 
$\inf_{h \in {\cal H}} \sum_{i=1}^N|h(x_i)-y_i|$ 
is not attained, we allow some small optimization error $\epsilon>0$ from an ERM oracle, namely allowing it to output any hypothesis whose loss is within $\epsilon$ from the infimum.}
\begin{align}
    \text{(ERM):}~~~~~ \arg \min_{h \in {\cal H}} \sum_{i=1}^N |h(x_i)- y_i|.   \label{eq:ERM}
\end{align}

The main goal of this paper is to advance our understanding of non-parametric learning in the more general  setting of {\em online learning}, which is general enough to capture a variety of other learning settings as special cases and  has found  applications in a diversity of fields, including optimization and {\em game theory}, which is also in the focus of this paper. We will consider a fairly general online learning setting wherein a learner interacts with an adversary over a number of rounds. In each round $t=1,2,\ldots, T$, the learner picks a distribution $\hat{\cal D}_t$ over functions $h: {\cal X} \rightarrow {\cal Y}$, where $\cal X$ is a feature set and ${\cal Y}\subseteq \mathbb{R}$ is a label set, then the adversary picks a feature-label pair $(x_t,y_t) \in {\cal X} \times {\cal Y}$, and then the learner draws a sample $\hat{h}_t \sim \hat{\cal D}_t$ and suffers loss $|\hat{h}_t(x_t)-y_t|$.
The learner's losses add up over rounds and the learner's goal is to make her total loss over several rounds as small as possible compared to some benchmark loss computed with hindsight information. 

There are many variations to the theme depending on what distribution $\hat{\cal D}_t$ the learner is allowed to use, what data $(x_t,y_t)$ the adversary is allowed to supply, what  benchmark loss the learner competes against, etc. In this work, we consider the common setting where there is some class $\cal H$ of hypotheses from $\cal X$ to $\cal Y$, and the performance of the learner is measured against the optimal hindsight error made by functions in this class: $\inf_{h \in {\cal H}} \sum_{t=1}^T|h(x_t)-y_t|$, i.e.~the goal of the learner is to minimize the following quantity, called {\em regret}, in expectation or with good probability:
\begin{align}
    \sum_{t=1}^T |\hat{h}_t(x_t)- y_t| - \inf_{h \in {\cal H}} \sum_{t=1}^T|h(x_t)-y_t|. \label{eq:regret costis}
\end{align}
When the distributions $\hat{\cal D}_t$ chosen by the learner are supported on hypotheses from $\cal H$ the learner is called {\em (randomized) proper}, otherwise the learner is called {\em improper}. When the pairs $(x_t,y_t)$ chosen by the adversary satisfy that $y_t=h(x_t)$ for some $h \in {\cal H}$ the setting is called {\em realizable} otherwise the setting is called {\em agnostic}. Finally, when ${\cal Y} = \{0,1\}$, we call $\cal H$ a {\em 0-1-valued or binary-valued concept class} and the  learning task {\em online classification}. {Otherwise, when ${\cal Y} = \mathbb{R}$ or some subset such as $[0,1]$}, $\cal H$ is called {\em real-valued} and the learning task is called {\em online regression}.

Similar to the stochastic setting, there has been extensive work {in the online setting} on developing complexity measures of concept classes, which suffice to characterize or bound the optimal regret~\eqref{eq:regret costis} that is attainable by a proper or improper learner, in the realizable or agnostic setting, and for binary-valued or real-valued functions. For example, the celebrated work of~\cite{littlestone_learning_1988} characterizes the optimal regret bound attainable by an improper learner, in the realizable online classification setting, in terms of a complexity measure of the concept class $\cal H$ that is now known as {\em Littlestone dimension} of $\cal H$; see Definition~\ref{def:littlestone}. More recent work by~\cite{daskalakis2022fast} provides near-matching bounds for the optimal regret of proper learners in the same setting, and generalizes this result to proper learners in the realizable online regression setting, providing regret bounds in terms of the natural generalization to this setting of Littlestone dimension, called {\em sequential fat-shattering dimension}; see Definition~\ref{def:sequential fat shattering}. In the agnostic setting, \cite{ben2009agnostic} {and \cite{alon2021adversarial}} characterize the optimal regret of improper learners in the online classification setting in terms of Littlestone dimension, while in the online regression setting \cite{rakhlin2010online} characterize the optimal regret of proper learners in terms of the sequential Rademacher complexity as well as in terms of the sequential fat-shattering dimension, the latter characterization being recently tightened by~\cite{block2021majorizing}. Finally, \cite{hanneke2021online} obtain near-optimal proper learners for online classification in the agnostic setting via more constructive arguments compared to~\cite{rakhlin2010online} and~\cite{block2021majorizing}.

In contrast to the stochastic setting, however, our understanding of attainable regret in the online learning setting is quite more limited, in the sense that the afore-described works, which bound or characterize the optimal regret, are either non-constructive or make oracle queries to the Standard Optimal Algorithm (SOA) proposed by~\cite{littlestone_learning_1988} or generalizations thereof~\citep{daskalakis2022fast}. 
{Indeed, learning algorithms for non-parametric hypothesis classes must have access to some oracle in order to interface with a potentially infinite menu of hypotheses. We must keep in mind, though, that our goal in studying learning in this general setting is to say something meaningful about specific learning tasks.  Thus, our selection of oracle model should be informed by the tasks to which we hope to apply our non-parametric algorithm.} SOA and its generalizations involve computing the Littlestone dimension or the sequential fat-shattering dimension of concept classes defined by the online learner in the course of its interaction with the adversary, which are challenging computations, even when the concept class and the set of features are finite~\citep{manurangsi2017inapproximability,manurangsi2022improved}. Thus, the non-parametric learning algorithms coming from the SOA oracle model are utterly useless for any practical applications.

{On the other hand, the stochastic learning setting is studied under the more standard ERM oracle model~\eqref{eq:ERM}. The learning algorithms here enjoy both success in their guaranteed performance and practical feasibility in their more realistic oracle assumption.  Ideally, we would like to construct online learning algorithms using this standard oracle too.}
One of the main questions we ask is thus the following:
\begin{goal}
In the non-parametric online learning setting, do there exist learning algorithms whose steps run in finite time given access to an ERM oracle as well as standard arithmetic operations, and whose regret is finite or sublinearly growing with $T$?
\end{goal}
One of our main contributions is to provide positive answers to this question for the online classification setting, as summarized in Table~\ref{table:summary of results}. In the realizable setting,  we provide an improper learner whose regret is finite and a proper learner whose regret grows sublinearly in the number of rounds; see Theorem~\ref{thm:online-realizable}. In the agnostic setting, we provide an improper learner with sublinear regret as well as a proper learner whose regret is also sublinear but grows faster than that of the proper learner; see Theorem~\ref{thm:agnostic}. The regret, time per iteration, and ERM oracle calls per iteration of our algorithms are bounded in terms of the Littlestone dimension of the concept class and/or the {\em threshold dimension} of the concept class, formally defined in Definition~\ref{def: threshold dimension} and related to the Littlestone dimension as per Lemma~\ref{lem:dim-relate-2}. We note that our algorithms use a weaker oracle than the ERM oracle, called {\em consistent oracle}, formally given in Definition~\ref{def:consistent}. This oracle takes as input a set of examples $(x_1,y_1),\ldots,(x_n,y_n)$ and outputs some $h \in \cH$ such that $y_i = h(x_i)$ for all $i$, if such $h$ exists.

{We note that, in the worst case, our algorithms require a number of iterations that is exponential in the Littlestone dimension.  This is expected due to the lowers bounds of~\cite{hazan2016computational}.  They show that}
there exist finite concept classes $\cH$ such that $\tilde{\Omega}(\sqrt{|\cH|})$-many ERM and function evaluation queries are necessary to obtain sublinear regret in the agnostic proper setting. When $\cH$ is finite, its Littestone dimension is bounded by $\log |\cH|$. Thus 
the total time in the last line of Table~\ref{table:summary of results} should have exponential dependence on the Littlestone dimension.

\begin{table}
\centering
\begin{tabular}{|c|c|c|c|}
\hline
Setting & Time per iter. & ERM calls/iter. & Regret \\
\hline 
Realizable, improper & $\min(t,4^{m},C^{d})$ & 1 &  $\min(4^{m},C^{d})$ \\
Realizable, proper & t & 1 & $T^{\frac{2m+2}{2m+3}}$ \\
Agnostic, improper &  $t^{\min\lr{4^{m},C^{d}}}$ & $t^{\min\lr{4^{m},C^{d}}}$  & $\sqrt{T \min\lr{4^{m},C^{d}}}$ \\
Agnostic, proper & $e^{o(T)}$ & $T^d$  & $T^{\frac{2m+3}{2m+4}}$ \\
\hline
\end{tabular}
\caption{The table describes the complexities of our online learner, Algorithm~\ref{alg:online-learner-re} (and its extension to the agnostic setting), in various settings, up to polylogarithmic factors. We use the following parameters: $m = \tr(\cH)$, the threshold dimension of the concept class; $d = \Lit(\cH)$, the Littlestone dimension of $\cH$;  $T$, the total number of iterations; and $t$, the current iteration count (in case the complexity is different for different $t$). The complexities are up to polylogarithmic factors and $C$ is an absolute constant.} \label{table:summary of results}
\end{table}

One of the main applications of online learning is for the purpose of equilibrium learning in games. Indeed, the existence of agnostic, proper learners whose regret grows sublinearly in the number of rounds in the online regression setting with finite concept classes can be used to establish the existence as well as the distributed learnability of minimax equilibria in two-player zero-sum games with a finite number of actions per player and coarse correlated equilibria in multi-player general-sum games with a finite number of actions per player; see e.g.~\cite{cesa-bianchi_prediction_2006}. This result has been recently generalized to non-parametric games, i.e.~games wherein players have an infinite set of actions, under the condition that a collection of concept classes (one class per player) defined in terms of the game's payoff matrix have finite Littlestone or sequential fat-shattering dimensions~\cite{hanneke2021online,daskalakis2022fast,rakhlin2010online}. However, the resulting algorithms for equilibrium learning in non-parametric games also involve SOA oracles. Our second goal in this paper is the following:
\begin{goal}
Consider a family of non-parametric two-player zero-sum (respectively multi-player general-sum) games for which minimax (respectively coarse correlated) equilibria exist. For such family of games, are there algorithms for computing an $\epsilon$-approximate minimax (respectively coarse correlated) equilibrium, which run in finite time given access to best-response oracles for each player (a.k.a.~ERM calls) as well as standard arithmetic operations?
\end{goal}
Our other main contribution is to provide positive answers to this question, as summarized in Table~\ref{table:summary of results games}. In particular, Theorem~\ref{thm:zero sum eq computation}  provides an algorithm computing an $\epsilon$-approximate minimax equilibrium of a two-player zero-sum game, whose number of iterations, time per iteration and number of best-response (a.k.a.~ERM) calls are bounded in terms of the {\em fat-threshold dimension of the game}, as per Definition~\ref{def:mat-form}. Theorem~\ref{thm:CCE} provides similar results for approximate coarse correlated equilibrium computation. Our results apply to both binary-valued and real-valued games. It is important to note that our algorithm for solving zero-sum games is a variant of the double oracle algorithm, proposed by~\cite{mcmahan2003planning}. In our variant, the players grow the action sets they consider in alternating rounds of the algorithm as opposed to simultaneously in every round. So our Theorem~\ref{thm:zero sum eq computation} provides conditions under which our variant of the double oracle algorithm converges in games with infinite action spaces, answering a question raised by~\cite{GempABBBCDVDEEH22} and their references. In a similar vein, our algorithm for solving multi-player games provides a multi-player variant of the double oracle algorithm and conditions under which it converges. Such multi-oracle algorithms are used extensively in practice for equilibrium computation in large games such as multi-agent reinforcement learning; see e.g.~the discussion on the policy-space response oracles (PSRO) algorithm in~\citep{GempABBBCDVDEEH22} and its references. Again our work provides convergence guarantees when the action sets are infinite.

{We want to highlight that we give the first algorithm for general concept classes and games \textit{that can be implemented} with access to an ERM oracle. Though SOA-oracle algorithms only require polynomially-many iterations in the Littlestone dimension, the execution of a single iteration for even simple tasks can take exponentially long.  Our ERM-oracle algorithms enjoy fast per-iteration time complexity, and often in terminate in far fewer than the worst-case exponentially-many iterations.  It is no surprise that the algorithms that arise from the ERM-oracle model parallel the algorithms actually used in the practice of solving large games (known as double-oracle algorithms; discussed under Goal 2).  In contrast, the algorithms that arise from the SOA-oracle model are utterly dissimilar from any practical algorithms.  The thesis of this work is that, rather than making oracle assumptions based on what will guarantee a polynomial regret bound, we should instead select an oracle based on what is practically feasible, and then from there, see what regret guarantees are possible.}

\begin{table}
\centering
\begin{tabular}{|c|c|c|c|}
\hline
Setting & Time per iter. & BR calls/it. & \#iterations \\
\hline 
Minmax, 0-1 valued & $t/\epsilon^4$ & $\log t/\epsilon^2$ & $C^{ \Lit(G)/\epsilon^2}\wedge \epsilon^{-C \VC(G)^2\tr(G)\log \VC(G)/\epsilon^4}$ \\
Minmax, real valued & $t/\epsilon^4$ & $\log t/\epsilon^2$ & $C^{ \sfat(G,\epsilon)/\epsilon^2}\wedge \epsilon^{-CI(G)^2\fatr(G,\epsilon)/\epsilon^5}$ \\
CCE, 0-1 valued & $kt/\epsilon^2$ & $k\log t/\epsilon^2$  & $C^{(k/\epsilon^3)\Lit(G)}
$  \\
CCE, real valued & $kt/\epsilon^2$ & $k\log t/\epsilon^2$  & $C^{(k/\epsilon^3)\sfat(G,\epsilon)}
$  \\
\hline
\end{tabular}
\caption{The table describes the time per iteration, the number of best-response calls per iteration and the number of iterations of our algorithms, up to polylogarithmic factors for finding an $O(\epsilon)$-approximate Nash in a zero-sum two player game (minmax equilibrium) and Coarse Correlated Equilibrium (CCE) in general games $G$. Here, $C>0$ is a universal constant, and $\Lit,\VC,\tr,\sfat,\fat,\fatr$ denote Littlestone, VC, threshold, sequential fat, fat and fat-threshold dimensions of $G$, $I(G)=\int_0^1\lr{\sqrt{\fat(G,\delta)d\delta}}^2$ and $\wedge$ denotes a minimum of two terms.} \label{table:summary of results games}
\end{table}

\section{Preliminaries} \label{sec:prelim-short}

We include below a shortened version of the preliminaries. See Section~\ref{sec:prelim} for a full version. 

\paragraph{Games}
A {\em $k$-player game} is a pair $(\actions, u)$, where $\actions=\mathcal{A}_1 \times \cdots \times \mathcal{A}_k$ and $ u=(u_1,\dots,u_k)$, where each $u_p \colon \actions \to \hR$.
Each $\mathcal{A}_p$ is the set of {\em actions} (a.k.a. \emph{strategies}) available to player~$p$ and each $u_p$ is the {\em utility, or payoff, function} of player $p$, which maps the set of {\em action profiles} $\cal A$ to the reals. Each player's goal is to maximize their own utility. We denote by $\cA_{-p}$ the Cartesian product~of~$\{\mathcal{A}_{j}\}_{q\ne p}$. Similarly, for any action $ a = (a_1,\dots,a_k)\in \actions$, denote by $a_{-p}$ the Cartesian product of $\{a_q\}_{q \ne p}$. A \emph{mixed strategy} for player $p$ is a distribution over $\cA_p$.
A \emph{zero-sum} game is a two-player game such that $u_1(a,b)=-u_2(a,b)$ for all $a\in \cA$ and $b\in \cB$. We sometimes compress our notation and represent a zero-sum game as $({\cal A},{\cal B},u)$ where $u:{\cal A}\times {\cal B} \rightarrow \hR$ is a single function representing the utility function of player~$2$. Player $2$ aims to maximize this utility while player $1$ aims to minimize this utility. 
\begin{definition}[$\epsilon$-Nash equilibrium and $\epsilon$-CCE]
Let $(\actions,u)$ denote a game. An $\epsilon$-approximate Nash equilibrium is a collection of probability measures, $\mu_1,\dots,\mu_k$, over $\cA_1,\dots,\cA_k$, respectively, such that for any player $p\in [k]$ and any $d_p \in \cA_p$,
\[
\E_{{a} \sim \mu_1 \times \cdots \times \mu_k} [u_p(d_p,{a}_{-p})]
\le 
\E_{ a  \sim \mu_1 \times \cdots \times \mu_k} [u_p(a)] - \epsilon.
\]
A \emph{Coarse Correlated Equilibrium} (CCE) is a joint measure $\mu$ over $\actions$ such that for any player $p$ and any $d_p \in \cA_p$,
\[
\E_{{a} \sim \mu} [u_p(d_p,{a}_{-p})]
\le 
\E_{a  \sim \mu} [u_p(a)] - \epsilon.
\]
\end{definition}

\begin{definition}[Minimax equilibrium]
Given a zero-sum game $G=(\cA, \cB,u)$\footnote{We slightly abuse notation and throughout denote a zero-sum game $(\cA \times \cB,u)$ as $(\cA, \cB,u)$}, we say that $G$ has a \emph{minimax equilibrium} if
\begin{equation}\label{eq:minmax}
\inf_{\mu_1 \in \Delta(\cA)}\sup_{\mu_2\in \Delta(\cB)} \E_{a\sim \mu_1,b\sim\mu_2}[u(a,b)]
= 
\sup_{\mu_2\in \Delta(\cB)} \inf_{\mu_1 \in \Delta(\cA)} \E_{a\sim \mu_1,b\sim\mu_2}[u(a,b)]
\end{equation}
where $\Delta(\cA)$ and $\Delta(\cB)$ denote the set of all probability measures over $\cA$ and $\cB$ respectively. If Eq.~\eqref{eq:minmax} is satisfied, we say the probability measures $\mu_1,\mu_2$ optimizing Eq.~\eqref{eq:minmax} are a \emph{minimax equilibrium} of $G$.  Denote by $\Val(G)$ \emph{value} of the game, which is the value of both sides of Eq.~\eqref{eq:minmax}.
\end{definition}

\paragraph{Dimensions}
Given a real-valued function-class $\cF$ over a domain $X$ and $\epsilon>0$, define the $\epsilon$-fat threshold dimension of $\cF$, $\fatr(\cF,\epsilon)$, to be the largest $m\ge 0$ such that there exist $f_1,\dots,f_m \in \cF$ and $x_1,\dots,x_m \in X$ and a threshold $\theta \in \mathbb{R}$ such that 
\begin{equation}\label{eq:fatrd-main}
    \begin{aligned} 
        f_i(x_j) &\geq \theta + \epsilon \qquad &\text{for all }i\leq j \in [d]\\
        f_i(x_j) &\leq \theta \qquad &\text{for all }i > j \in [d]
    \end{aligned}
    \end{equation}
For 0-1 classes define the \emph{threshold dimension} of $\cF$ as $\tr(\cF)=\fatr(\cF,1/2)$. We will also use the notion of VC and Littlestone dimension, $\VC(\cF)$ and $\Lit(\cF)$ for real 0-1 valued classes and their real-valued analogues, the fat and sequential fat-shattering dimensions, $\fat(\cF,\epsilon)$ and $\sfat(\cF,\epsilon)$ (all defined in Section~\ref{sec:prelim}). 

These dimensions can be extended to games: For each player $p$, her utility $u_p: \cA_p \times \cA_{-p} \to \hR$ can be thought of as a concept class $\cF_p$ over the domain set $X_p = \cA_p$, whose concepts $f_{a_{-p}}$ are parametrized by elements $a_{-p} \in \cA_{-p}$ and are defined by $f_{a_{-p}}(a_p) := u_p(a_p,a_{-p})$, for each $a_p \in \cA_p$. We define the dimension of a game to be the maximal dimension over these utility function classes, where $p$ ranges across all players.

\paragraph{Convex hulls and their dimensions.} Denote by $\conv(\cF)$ the \emph{convex hull of $\cF$}, namely the class of all convex combinations of elements from $\cF$, by $\Delta(\cX)$ the set of all probability distributions over $\cX$, by $\dconv(\cF)$ the \emph{dual convex hull of $\cF$} the set of all elements from $\cF$ extended to the domain $\Delta(\cX)$ by taking an expectation, namely, $f(\mu) = \E_{x\sim \mu}[f(x)]$ for all $f\in \cF$ and $\mu\in \Delta(X)$, and by $\conv2(\cF) = \dconv(\conv(\cF))$. Similarly, for a game $G$ we denote by $\conv(G)$ the game obtained from $G$ where the action sets $\cA_p$ are replaced by $\Delta(\cA_p)$.
We prove the following theorem, which bounds the threshold dimension of $\conv2(\cF)$ (see \cref{sec:thresh-dim-mixed-proof} for the proof):

\begin{theorem}\label{thm:conv-bnd}
    For $[0,1]$-valued concept classes $\cF$, $\fatr(\conv(\cF),\epsilon) \le e^{C\sfat(\cF,\epsilon/C)/\epsilon^2}$ and same result holds when $\conv$ is replaced by $\dconv$ and $\convtwo$ and when $\cF$ is replaced by a zero-sum two-player game $G$. 
    Moreover,
    for a zero-sum two-player game $G$, $\fatr(\conv(G),\epsilon) \le \epsilon^{-CI(G)^2 \fatr(G,\epsilon/C)/\epsilon^5}$,
    where $I(G) := \lr{\int_0^1 \sqrt{\fat(G,\epsilon)}d\epsilon}^2$ and $C>0$ is a universal constant. For $\lrset{0,1}$-valued games, we have $
\tr(\conv(G),\epsilon) \le O\left( (1/\epsilon)^{C\VC(G)^2\tr(G)\log(\VC(G))}\right)$ 
\end{theorem}
Throughout, we use the notation $\Tilde{O}(\cdot)$ for omitting poly-logarithmic factors.

\section{Online learning}\label{sec:online-learning}

We describe below our algorithm and results in the online learning setting. First, we define the oracles that the algorithm is allowed to use:
\begin{definition}[Consistent oracle]\label{def:consistent}
For any set of pairs $\lrset{(x_1, y_1),\dots,(x_t,y_t)}$, the oracle outputs $h \in \cH$
such that $h(x_i) = y_i$, for all $i
\in[t]$, if exists (otherwise it is undefined).
\end{definition}

\begin{definition}[Value oracle] \label{def:value}
    For any $h \in \cH$ and $x \in X$, return $h(x)$.
\end{definition}
We notice that the only access that the algorithm has to the function-class $\cH$ is via the consistent oracle. Further, the only access to the set $X$ is via the examples generated by the adversary. We present the following theorem on Algorithm~\ref{alg:online-learner-re} which only has access to these two oracles:

\begin{theorem}[Realizable]\label{thm:online-realizable}
Let $\mathcal{\cH}$ be a 0-1 valued concept class and assume the stream of examples is realizable by some $h^* \in \cH$. Then, Algorithm~\ref{alg:online-learner-re}, instantiated as an improper learner, has the following bound on its regret and on its number of calls to the consistent oracle: $\min\lr{\cO\lr{4^{\tr(\cH)}}, e^{\cO\lr{\Lit(\cH)}}}$; if Algorithm~\ref{alg:online-learner-re} is instantiated as a randomized proper learner, the bound changes to $\widetilde{O}\lr{T^{\frac{2\tr(\cH)+2}{2\tr(\cH)+3}}}$. Here, $\tr(\cH)$ and $\Lit(\cH)$ are the threshold and Littlestone dimensions of $\cH$, respectively.
\end{theorem}
We notice that Theorem~\ref{thm:online-realizable} provides a new proof that $\Lit(\cH) \le \cO\lr{4^{\tr(\cH)}}$: indeed, Littlestone dimension equals the smallest regret possible for any improper learner in the realizable setting.\footnote{Equivalently, a constant bound on the regret is called \emph{mistake bound}.} Notice that the best known bound is $\Lit(\cH) \le 2^{\tr(\cH)}$ \citep[Theorem 3]{alon2019private} \citep{hodges1997shorter,shelah1990classification}. Additionally, as a direct corollary, we obtain the first polynomial-time algorithm that, given a full description of the class, implements a no-regret learner whose mistake bound depends only on Littlestone's dimension.
\begin{corollary}
    There is an online learner who has access to a table of size $n=|\mathcal{H}||\mathcal{X}|$ that describes a binary-valued concept-class $\cH$ over $\cX$, that has a mistake bound of $e^{\cO(\Lit(\cH))}$ in the realizable setting and runs in time $O(n)$ per iteration.
\end{corollary}
For comparison, SOA achieves a mistake bound of $\Lit(\cH)$, however, its runtime is not polynomial in $n$: while SOA requires computation of Littlestone's dimension, under hardness assumptions, it is impossible to even approximate Littlestone's dimension in time polynomial in $n$ as long as it is $\omega(1)$ \citep{frances1998optimal,manurangsi2017inapproximability,manurangsi2022improved}. On the other hand, the \emph{halving} algorithm \citep{shalev2014understanding} takes time $O(|\cH|)$ per iteration, however, its mistake bound is $O(\log |\cH|)$ -- this does not depend only on $\Lit(\cH)$.

Next, we describe the result for the agnostic setting, which is obtained by applying the reduction of \cite{ben2009agnostic} from the agnostic to the realizable setting, while using Algorithm~\ref{alg:online-learner-re} as the realizable learner in this reduction (see Section~\ref{app:online learning} for the proof):
\begin{theorem}[Agnostic]\label{thm:agnostic}
    Let $\cH$ be a 0-1 valued class. Then, there exists an improper learner which accesses $\cH$ only via the consistent and value oracles, that achieves in the improper setting a regret of $\sqrt{T \min\lr{\widetilde{\cO}\lr{4^{\tr(\cH)}},e^{\cO(\Lit(\cH))}}}$ and $\widetilde\cO\lr{T^{\frac{2\tr(\cH)+3}{2\tr(\cH)+4}}}$ in the proper setting.
\end{theorem}

\subsection{Algorithm}
Below we describe the algorithm for the realizable setting. It has two variants, proper and improper. For convenience of notation, in the proper setting, we say that the algorithm plays a distribution $\mu^t$ over hypotheses $h\in \cH$ in each iteration $t$ and suffers loss $\Loss(\mu^t,(x_t,y_t)) := \Pr_{h \sim \mu^t}[h(x_t)\ne y_t]$.\footnote{The alternative would be to sample one $h$ from $\cH$ - these two notions are equivalent if one is interested in an expected loss.} In the improper setting, the algorithm is allowed to take a weighted majority vote over hypotheses and it will select the label that is predicted with the largest probability according to $\mu$. Formally, the learner predicts $\hat{y}_t :=\Maj(\mu^t,x_t)=\begin{cases}
    1, & \Pr_{h \sim \mu^t}[h(x_t)=1] \geq 1/2\\
    0, & \text{otherwise}
\end{cases}$.

The algorithm proceeds in phases, where each phase consists of multiple rounds of prediction. In each phase, $j$, the algorithm plays a distribution over some pool of actions $\{h_1,\dots,h_j\}$. By the end of each phase, the algorithm adds a new action to the pool, $h_{j+1}$, which is taken as a hypothesis that is consistent with the whole history of elements $(x,y)$ observed by the algorithm throughout all phases. 

Next, we describe the phases. Fix a phase $j$, denote the distribution played by the algorithm at any round $t$ of this phase by $\mu^t$ and we describe how to determine $\mu^t$: first, $\mu^1$ is the uniform distribution over all the hypotheses available in this phase: $\{h_1,\dots,h_j\}$. In rounds $t$ when $\Loss(\mu^t,(x_t,y_t))\ge\epsilon$, the algorithm updates $\mu^t$ via a multiplicative-weight update. In the remaining rounds, no update is made, and $\mu^{t+1} \gets \mu^t$. We call this type of update \emph{Lazy multiplicative weights}. The phase $j$ ends once $T_j$ updates have been made. We rely on an auxiliary parameter $\alpha$ to determine the value of $T_j$. See Algorithm~\ref{alg:online-learner-re} for the main pseudocode and Algorithm~\ref{alg:lazy-MW} for the implementation of each phase.

    \begin{algorithm}[H]
    \caption{\texttt{Online Action Insertion}}\label{alg:online-learner-re}
    \textbf{Input:} A function class $\cH$. \\
    \textbf{Parameters:} $\epsilon,\alpha > 0$.\\
    \textbf{Subroutines:} 
    Consistent oracle (Definition~\ref{def:consistent}), Lazy Multiplicative Weights (Algorithm~\ref{alg:lazy-MW}).
    \begin{enumerate}
        \item Initialize active action set $\cH_1 \leftarrow \{h_1\}$ for arbitrary $h_1 \in \cH$.
        \item \textbf{For} phase $j=1,2,\ldots,J$:
        \begin{enumerate}
            \item 
            Instantiate the algorithm {\tt Lazy Multiplicative Weights} with the function-class $\cH_j$ and parameters $T_j= 
            \l\lceil \frac{C \log |\cH_j|}{\alpha^2} \r\rceil$  and $\epsilon$ ($C>0$ is a constant), in order to predict in the next classification rounds, until the execution of {\tt Lazy Multiplicative Weights} terminates.
            \item Call the consistent oracle to obtain $h_{j+1}\in\cH$ that is consistent with all pairs $(x,y)$ observed in all previous rounds throughout all phases. 
            \item Update the active action set $\cH_{j+1} \leftarrow\cH_j\cup\lrset{h_{j+1}}$.
        \end{enumerate}
    \end{enumerate}
\end{algorithm}

\begin{algorithm}[H]
    \caption{\texttt{Lazy Multiplicative Weights}}\label{alg:lazy-MW}
    \textbf{Input:} A finite set of functions $\cG=\lrset{g_1,\ldots,g_n}$ $\subseteq \lrset{0,1}^\cX$.\\
    \textbf{Parameters:} $K$ rounds with an update, accuracy parameter $\epsilon>0$, learning rate $\eta=\sqrt{\frac{\log\lrabs{\cG}}{2K}}$.
    \begin{enumerate}
        \item Initialize a uniform probability distribution $\mu^1=\l(\mu_1^1,\dots,\mu_{n}^1\r)=\lr{\mu^1(g_1),\dots,\mu^1(g_n)}$ over $\cG$, where $\mu_i^1 \gets \frac{1}{n}$ for all $i=1,\dots,n$.
        \item $k \leftarrow 0$.
        \item \textbf{For} $t=1,2,\dots$: 
        \begin{enumerate}
            \item \textbf{If proper:} predict $\mu^t$, observe $(x_t,y_t)$ and suffer $\Loss(\mu^t,(x_t,y_t))$.
            \\
            \textbf{Else (if improper):} Observe $x_t$, predict $\hat{y}_t :=\Maj(\mu^t,x_t)$, observe $y_t$ and suffer a loss $\hI\lr{\hat y_t \ne y_t}$. 
            \item \textbf{If} $\Loss(\mu^t,(x_t,y_t)) \ge \epsilon$:
            \begin{enumerate}
            \item 
                $\forall i,~
                   \mu^{t+1}_i \gets \mu^{t}_i \exp\lr{- \eta \lrabs{g_i(x_t)-y_t}}/\sum_{j=1}^n \mu^{t}_j \exp\lr{- \eta \lrabs{g_j(x_t)-y_t}}$.
            \item $k\gets k+1$.
            \item \textbf{If} $k = K$ then \textbf{Return}.
            \end{enumerate}
            \textbf{Else}: make no update: $\mu^{t+1} \leftarrow \mu^{t}$.
        \end{enumerate}
    \end{enumerate}
\end{algorithm}

\subsection{Proof}
\paragraph{Proper learner, bound in terms of $\tr(\cH)$}
We say that a distribution $\mu$ over $\cH$ makes an $\epsilon$-mistake on $(x,y)\in X\times \{0,1\}$ if $\Loss(\mu,(x,y)) \ge \epsilon$. We say that the algorithm makes an $\epsilon$-mistake on iteration $t$ if its prediction $\mu^t$ makes an $\epsilon$-mistake on $(x_t,y_t)$.
Our first goal would be to bound the number of $\epsilon$ mistakes made by the algorithm. This will be done by bounding the number of phases observed by the algorithm. Assume that the algorithm has observed more than $J$ phases, and for every $j \in [J]$,
denote by $Z_j$ the set of elements $(x,y)$ observed by the algorithm on phase $j$, on which the algorithm made an $\epsilon$-mistake. For any $h\in \cH$, denote $\Loss(h,Z_j) = \frac{1}{|Z_j|}\sum_{(x,y)\in Z_j} \hI\lr{h(x)\ne y}$ as the fraction of mistakes of $h$ on the elements of $Z_j$. Recall the pool of functions $h_1,\dots,h_J$ maintained by the algorithm in phase $J$. In the following lemma, we argue that the loss applied on the functions $h_1,\dots,h_J$ and the sets $Z_1,\dots,Z_J$, have the following threshold behavior:
  \begin{lemma}\label{lem:eps-frac}
    Let $h_1,\dots,h_J$ and $Z_1,\dots,Z_J$ be defined as in the preceding paragraph. Then, for all $i,j\in[J]$
    \begin{equation} \label{eq:pr-threshold}
    \begin{aligned}
    \Loss(h_i,Z_j) \ge \epsilon-\alpha & \qquad \text{if } i \le j \\
    \Loss(h_i,Z_j) = 0 & \qquad \text{if } i > j
    \end{aligned}
    \end{equation}
\end{lemma}  
\begin{proof}
    For $i > j$, $\Loss(h_i,Z_j)=0$ since $h_i$ is the output of a consistent oracle that observed all the examples from previous rounds, including those from $Z_j$. Next, we proceed with $i \le j$.
    To prove for $i \le j$, fix round $j$,
    and notice that the Lazy multiplicative weights algorithm, restricted to rounds where the adversary played actions in $Z_j$, behaves exactly as the original multiplicative weights algorithm. 
    There are $T_j$ such rounds, and let $\mu^{t_1},\ldots,\mu^{t_{T_j}}$ be the learner's distributions over $\cH_j$ for these rounds. 
 From the Multiplicative Weights guarantee (see Lemma~\ref{lem:multiplicative}) and by definition of $T_j$, the regret of the learner is upper bounded by $\cO\lr{\sqrt{T_j \log j}} \le T_j\alpha$:
    \begin{align*}
    \epsilon T_j
    &\le
    \sum_{k=1}^{T_j} \Loss(\mu^{t_k},(x_{t_k},y_{t_k}))
    \le 
    \min_{h \in \{h_1,\dots,h_j\}} \sum_{k=1}^{T_j} \lrabs{h(x_{t_k})-y_{t_k}}
    + 
    T_j\alpha \\
    &= T_j \min_{h \in \{h_1,\dots,h_j\}} \Loss(h,Z_j) + T_j \alpha,
    \end{align*}
    where the first inequality holds
    since the algorithm makes an $\epsilon$-mistake in each of these rounds, the second due to the regret guarantee of multiplicative weights and the third since $Z_j = \{(x_{t_k},y_{t_k})\}_{k\in[T_j]}$.
    By rearranging terms, this concludes that for all $i \le j$, $\Loss(h_i,Z_j) \ge \epsilon-\alpha$. 
\end{proof}

Let $h^\star$ be the function that is realizable with all the examples provided by the adversary throughout all phases, namely, $h^\star(x_t)=y_t$ for observed $(x_t,y_t)$. Such $h^\star$ exists due to the realizability assumption. For $i \in [J]$,
define the function $f_i(x) = (h_i \oplus h^\star)(x) := \hI(h_i(x) \neq h^\star(x))$, and notice that Eq.~\eqref{eq:pr-threshold} implies the following, where $\unif(Z_j)$ denotes a uniform distribution over $Z_j$ and only $x$ is sampled (rather than the pair $(x,y)$):
\begin{equation}\label{eq:loss-thresh}
\begin{aligned}
\E_{x\sim \unif(Z_j)} [f_i(x)] \ge \epsilon-\alpha & \qquad \text{if } i \le j \\
\E_{x\sim \unif(Z_j)} [f_i(x)] = 0 & \qquad \text{if } i > j
\end{aligned}
\end{equation}

By a simple inductive argument, Eq.~\eqref{eq:loss-thresh} is sufficient to imply a lower bound on the threshold dimension of $\{f_1,\dots,f_J\}$ in terms of $J$, which equivalently bounds the number of phases $J$ in terms of $\tr\lr{\{f_1,\dots,f_J\}}$:
\begin{lemma}\label{lem:thresh-game}
    Let $f_1,...,f_J \colon X \to \{0,1\}$ and $Z_1,...,Z_{J}  \subset X$ such that Eq.~\eqref{eq:loss-thresh} holds, let $m\in \mathbb{N}$ and assume that $J \ge \sum_{k=1}^m\lr{\frac{1}{\epsilon-\alpha}}^k$. Then, $\tr(\{f_1,\dots,f_J\}) \ge m$.
\end{lemma}
\begin{proof}
We prove the claim by induction on $m$.
For $m=0$, trivially there exists a threshold game of size $0$. For the induction step, assume that the claim holds for $m-1$. To prove for $m$, we need to show that there exist functions $g_1,\dots,g_m \in \{f_1,\dots,f_J\}$ and $x_1,\dots,x_m\in X$ such that $g_i(x_j) = \hI(i\le j)$. We will start by determining $x_m$ and $g_m$ and the remaining elements $f_i$ and $g_i$ will be taken from the induction hypothesis. To determine $x_m$, we notice that from double summation and from the condition of this lemma,
\begin{align*}
\E_{x\sim\unif(Z_J)} \lrabs{\lrset{i\in [J] \colon f_i(x)=1}}
&= 
\sum_{i=1}^J \E_{x\sim\unif(Z_J)}[f_i(x)] \ge J(\epsilon - \alpha)\enspace.
\end{align*}
This implies that there exists $x' \in Z_J$ such that the set $I_{x'}:= \lrset{i\in [J] \colon f_i(x')=1}$ is of cardinality at least $J(\epsilon-\alpha)$. We take $x_m=x'$. The functions $g_1,\dots,g_m$ will be taken from $\{g_i\}_{i\in I_{x'}}$ and this will guarantee that $g_i(x_m) = 1 = \hI(i\le m)$. We denote $i' = \max I_{x'}$ and set $g_m = f_{i'}$. The remaining elements $x_1,\dots,x_{m-1}$ will be taken from the sets $\{Z_j\}_{j\in I'}$ where $I' = I_{x'}\setminus \{i'\}$. This will guarantee that for all $j <m$, $g_m(x_j) = 0 = \hI(m \le j)$, using Eq.~\eqref{eq:loss-thresh}. We would like to apply the induction hypothesis on the functions $\{f_i\}_{i \in I'}$ and the sets $\{Z_j\}_{j\in I'}$. Indeed, the induction hypothesis can be applied since $|I'|=|I_{x'}|-1 \ge (\epsilon-\alpha) J - 1\ge \sum_{k=1}^{m-1}\lr{\frac{1}{\epsilon-\alpha}}^k$, which follows by the computed bound on $|I_{x'}|$ and by the assumption $J\ge \sum_{k=1}^m \lr{\frac{1}{\epsilon-\alpha}}^k$. Hence, the induction hypothesis yields that $\tr(\{f_i\}_{i \in I'}) \ge m-1$, which imply the existence of functions $g_1,\dots,g_{m-1} \in \{f_i\}_{i \in I'}$ and elements $x_1,\dots,x_{m-1} \in X$ such that $g_i(x_j) = \hI\lr{i\le j}$ for all $i,j \in [m-1]$. Together with $g_m$ and $x_m$, the arguments above imply that $g_i(x_j) = \hI\lr{i\le j}$ for all $i,j\in [m]$, which concludes the induction step. The proof follows.
\end{proof}

We are now ready to bound the number of $\epsilon$-mistakes of the algorithm.
Combining Lemma~\ref{lem:eps-frac}, Eq.~\ref{eq:loss-thresh} and Lemma~\ref{lem:thresh-game}, we obtain that 
$J < \sum_{k=1}^{\tr\lr{\lrset{f_1,\dots,f_J}}+1} \lr{\epsilon-\alpha}^{-k} \le 2 \lr{\epsilon-\alpha}^{-\tr\lr{\lrset{f_1,\dots,f_J}}-1}$,
assuming that $\epsilon \le 1/2$.
Recall that $f_i=h_i \oplus h^\star$ and define $\cH\oplus h^\star := \lrset{h\oplus h^\star \colon h\in \cH}$. Lemma~\ref{lem:xoring} argues that $\tr(\cH\oplus h^\star) \le 2\tr(\cH)+1$, consequently, $J \le 2\lr{\epsilon-\alpha}^{-2\tr(\cH)-2}$. Recall that the number of $\epsilon$-mistakes in each phase $j$ is at most $\cO\lr{\log(j)/\alpha^2}$. Assuming that there are at least $J$ phases, the total number of mistakes is bounded by $\cO\lr{J \log J/\alpha^2}$. Substituting the bound on $J$ and optimizing over $\alpha$, one obtains that the number of $\epsilon$-mistakes is at most $\cO\lr{\tr(\cH)^3\epsilon^{-2\tr(\cH)-2}}$ (see Lemma~\ref{lem:optimize-alpha}). The total regret of the algorithm is bounded by the number of $\epsilon$-mistakes, plus $\epsilon T$, in order to account for the loss for less-than-$\epsilon$-mistakes. Setting $\epsilon=T^{-1/(2\tr(\cH)+3)}$ yields the bound of $\widetilde\cO\lr{T^{\frac{2\tr(\cH)+2}{2\tr(\cH)+3}}}$.
\paragraph{Bounds for the improper learner.}
Recall that the improper learner takes a majority vote over $\mu^t$, therefore, it makes a mistake only if $\mu^t$ makes a $1/2$-mistake. Substituting $\epsilon=1/2$ in the bound on the number of $\epsilon$-mistakes of the proper learner, one obtains the desired bound of $\tilde\cO\lr{4^{\tr(\cH)}}$. For the bound in terms of Littlestone's dimension, we use the fact that Eq.~\eqref{eq:loss-thresh} implies that $\fatr(\dconv(\cH\oplus h^\star),\epsilon-\alpha) \ge J$ (for the dual convex $\dconv(\cdot)$ see Section~\ref{sec:prelim-short}). Indeed, the functions $f_1,\dots,f_J \in \cH\oplus h^\star$ and the elements $\lrset{\sum_{x\in Z_j}\frac{x}{|Z_j|}}_{j\in[J]} \in \conv(X)$, satisfy Eq.~\eqref{eq:fatrd-main} in the definition of the $\epsilon$-fat threshold dimension. The bound in Theorem~\ref{thm:conv-bnd} on the threshold dimension of the dual convex concludes the proof. See Lemma~\ref{lem:improp-lit} for the proof of this argument.

\section{Computing equilibria in games}

In this section, we will provide an algorithm for computing approximate Nash equilibria in zero-sum two-player games and coarse correlated equilibria in general multi-player games. The algorithm is allowed to use the following oracles:
\begin{definition}[Best-response and value oracles]
    An $\epsilon$-\emph{best response} oracle in a game $(\cA_1,\dots,\cA_k,u)$ receives a player $p$, a finitely-supported distribution $\mu_{-p}$ over the Cartesian product $\cA_{-p}=\prod_{q\ne p} \cA_q$ and outputs an action $\hat a_p \in \cA_p$ that $\epsilon$-maximizes the player's utility against a random sample from the distribution:
    \[
    \E_{a_{-p}\sim \mu_{-p}} \lrbra{u_p(\hat{a}_p, a_{-p})} \ge \sup_{a_p \in \cA_p} \E_{a_{-p}\sim \mu_{-p}} \lrbra{u_p(a_p, a_{-p})} - \epsilon\enspace.
    \]
    A \emph{value oracle} receives a player $p$ and actions $(a_1,\dots,a_k)$ and outputs $u_p(a_1,\dots,a_k)$.
\end{definition}
We notice that a \emph{best-response oracle} can be viewed as an ERM oracle, that maximizes reward instead of minimizing loss. We note that the only access that the algorithms have to the game is via these oracles and they are not allowed to access the action sets apart via these oracles. This aims to capture the scenario that their action sets are large or perhaps even infinite. Algorithms are given under the assumption that the sequential fat-shattering dimension $\sfat(G,\epsilon)$ of the game is finite. It has been shown that if the dimension is infinite, an equilibrium might not exist \citep{hanneke2021online, daskalakis2022fast}.\footnote{There are some delicacies in the statement of when there exists, or there does not exist, an equilibrium and the exact conditions are not known. See Section~\ref{sec:learnability-existence} and \citep{hanneke2021online, daskalakis2022fast} for more discussion.} In the following two sections, we present our two results: in Section~\ref{sec:nash-short} we present the result for computing a Nash equilibrium in two-player zero-sum games. In Section~\ref{sec:CCE-main} we study general sum games and show how to compute a coarse correlated equilibrium, whereas computing Nash is a significantly harder problem: for finite games, it is PPAD-hard \citep{daskalakis2009complexity,chen2009settling} whereas CCE is poly-time computable \citep{hart2000simple,hart1989existence}

\subsection{Approximating Nash equilibrium in zero-sum two-player games}\label{sec:nash-short}

\begin{theorem} \label{thm:zero sum eq computation}
Let $G = (\cA,\cB,u)$ be a zero-sum two-player game, where $u \colon \cA\times \cB \to [0,1]$ and let $\epsilon>0$. There is an algorithm to find an $O(\epsilon)$-Nash for this game using the following number of $\epsilon$-best response oracle calls (assuming this number is finite):
\[
\cO\lr{\min\lr{e^{C\sfat(G,\epsilon/C)/\epsilon^2},~ (1/\epsilon)^{CI(G)^2 \fatr(G,\epsilon/C)/\epsilon^5}}} ;~ I(G) = \lr{\int_0^1 \sqrt{\fat(G,\epsilon)}}^2.
\]
\end{theorem}
As a first step, we argue that it is possible to find an approximate Nash for a game $(\cA,\cB,u)$ where one of the players has a finite action set and random access to that set and the second player has an infinite action-set and an $\epsilon$-best response oracle. We use the reduction from online learning to equilibrium computation which states that if two players play an algorithm whose regret behaves as $o(T)$ in a zero-sum game, then the pair of uniform distributions over their actions converges to a Nash equilibrium as $T \to \infty$. We use the common technique where the player with the finite number of actions can play a no-regret algorithm (such as multiplicative weight update) and the second player plays best-response, which is a no-regret algorithm as well. We notice that a modification of this technique does not work when both players play best response since best-response is no-regret only if it is played second, after observing the opponent's action.

Next, we will provide the algorithm to compute an $\epsilon$-Nash in a zero-sum game $(\cA,\cB,u)$ where both $\cA$ and $\cB$ are possibly infinite or very large. The algorithm was inspired by the proof of existence of minimax equilibria by \cite{hanneke2021online}.
It gradually accumulates actions for each of the two players until reaching a sufficiently large finite subgame whose $\epsilon$-Nash equilibrium approximates the Nash of the complete game. In particular, at each iteration $t$, Player $1$ will hold a finite subset of actions, $A_t \subseteq \cA$, that grows as $t$ increases, namely, $A_0 \subseteq A_1 \subseteq A_2 \subseteq \cdots$. Similarly, Player $2$ will hold finite sets $B_0 \subseteq B_1 \subseteq \cdots$. The initial sets $A_0$ and $B_0$ are of cardinality $1$ and contain a single arbitrary element from $\cA$ and $\cB$, respectively. Then, the players take turns adding actions to change the value of the game in their favor. 
In particular each iteration $t$ begins where Players $1$ and $2$ hold the sets of actions $A_{t-1}$ and $B_{t-1}$, respectively. Then, Player $1$, whose aim is to minimize the utility $u$ and the value $\Val$, finds a set of actions $A_t \supseteq A_{t-1}$, such that $\Val(A_{t},B_{t-1}) \le \Val(A_{t-1},B_{t-1}) - \Omega(\epsilon)$. This is done by computing an approximate Nash equilibrium for the game $(\cA,B_{t-1},u)$ where Player $1$ is unrestricted and player 2 is restricted to her finite set $B_t$, and adding the support of this approximate Nash to $A_{t-1}$, thus creating $A_t$. Then, similarly, Player $2$ responds by finding a set of actions, $B_t \supseteq B_{t-1}$ that increase the value of the game, namely, $\Val(A_t,B_t) \ge \Val(A_t,B_{t-1}) + \Omega(\epsilon)$ thus changing the value in her favor. The algorithm stops when no player can improve the value by more than $\epsilon$. See Algorithm~\ref{alg:nash} for the pseudocode.

\begin{algorithm}[H]
\caption{\texttt{$\epsilon$-approximate Nash Equilibrium for a zero-sum game}}
\label{alg:nash}
\textbf{Input:} A zero-sum game $(\cA,\cB,u)$, a parameter $\epsilon>0$. \\
\textbf{Subroutines:}
\begin{itemize}
    \item \texttt{Nash}: Receives sets of actions $A$ and $B$ for both players, and an $\epsilon>0$, where either $A$ is finite or $B$ is. It outputs an $\epsilon$-Nash for the subgame $(A,B,u)$ (Algorithm~\ref{alg:nash-half-inf})
    \item $\Val$: Receives finite sets of actions of the players. Returns the value of this finite subgame.
\end{itemize}
\begin{enumerate}
    \item $A_0 \gets \{a\}$, $B_0 \gets \{b\}$, where $a\in \cA$ and $b \in \mathcal{B}$ are arbitrary actions.
    \item \textbf{For} $t=1,2,\dots$
    \begin{enumerate}
        \item $(\mu^{t,1},\mu^{t,2}) \gets$ \texttt{Nash}$\l(\cA,B_{t-1},\epsilon\r)$.
        \item $A_t \gets A_{t-1} \cup \texttt{Support}(\mu^{t,1})$.
        \item  $(\xi^{t,1},\xi^{t,2}) \gets \texttt{Nash}\l( A_t,\cB, \epsilon\r)$.
        \item $B_t \gets B_{t-1}\cup \texttt{Support}(\xi^{t,2})$.
        \item \textbf{If} 
        $\Val(A_t,B_{t-1}) \ge \Val(A_{t-1},B_{t-1}) - \epsilon$ 
        \textbf{or} 
        $\Val(A_t,B_{t}) \le \Val(A_t,B_{t-1}) + \epsilon$:
        \begin{enumerate}
            \item \textbf{Return} ($\xi^{t,1},\mu^{t,2}$).
        \end{enumerate}
    \end{enumerate}
\end{enumerate}
\end{algorithm}

The first statement that is proven is that the output of the algorithm is an approximate Nash equilibrium. Intuitively, this follows from the fact that when the game ends, no player can add action to drastically change their value. The second statement is that the algorithm eventually ends. This is proven using the fact that  $\fatr(\conv(G),\epsilon)$ is finite, which follows from Theorem~\ref{thm:conv-bnd}. Intuitively, the finiteness of $\fatr(\conv(G),\epsilon)$ implies that there exists no sequence of mixed strategies (i.e. distributions over actions) $\mu_1,\mu_2,\dots$ for player 1 and $\xi_1,\xi_2,\dots$ for player 2 and a threshold $\theta$ such that $u(\mu_i,\xi_j) \ge \theta+\epsilon$ if $i \le j$ and $u(\mu_i,\xi_j) \le \theta$ if $i>j$ -- this can be shown to imply that the players cannot keep adding actions to increase the value by $\Omega(\epsilon)$ to their favor indefinitely.

Lastly, we explain how the two bounds on $\fatr(\conv(G),\epsilon)$ are derived in Theorem~\ref{thm:conv-bnd}: the bound in terms of the threshold dimension of $G$ is an adaptation of the beautiful technique of \cite{hanneke2021online} using Ramsey numbers, that was used to prove the existence of minimax in a wide class of zero-sum games. The bound in terms of the sequential fat-shattering dimension uses a standard technique of comparison to the sequential Rademacher complexity of \cite{rakhlin2010online}.
\subsection{CCE in multi-player general-sum games}\label{sec:CCE-main}
We prove the following theorem for finding a CCE in a general game:
\begin{theorem}\label{thm:CCE}
Let $G=(\actions=\mathcal{A}_1 \times \cdots \times \mathcal{A}_k,u=(u_1,\cdots,u_k))$ be a multi-player game. Assume that utilities are bounded $u_p: \cA \to [0,1]$, and let $\epsilon>0$. Then, Algorithm~\ref{alg:CCE} executed with parameters $G,\epsilon$ will compute an $O(\epsilon)-\CCE$ for the game using using the following number of $\epsilon$-best response oracle calls:
$
\cO\lr{
e^{C(k/\epsilon^3)\sfat(G,\epsilon/C)} 
}
$.
\end{theorem}

In order to extend our algorithm for finding a Nash in a two-player zero-sum for our setting, we use the brilliant reduction of \cite{papadimitriou2008computing} from multiplayer games to a two-player game which we term \emph{the $\CCE$ game}.
Here, player $1$ selects an entire strategy profile $a$ and Player 2 selects a player $p$ and an alternative action $d_p$ for player $p$. The utility for player 2 in the $\CCE$ game, corresponds to the gain in utility (in the original game) made by player $p$ deviating to action $d_p$ when everyone is playing according to strategy profile $a$. This utility can be defined formally as a matrix whose entries are indexed by $a$ and $(p,d_p)$, as follows:
\begin{definition}[The CCE matrix of a game]\label{def:CCEM}
For a game $G=(\actions=\mathcal{A}_1 \times \cdots \times \mathcal{A}_k,u=(u_1,\cdots,u_k))$, the \emph{$\CCE$ matrix} $\MGS{\CCE}: \cA \times \lr{\bigcup_p \cA_p}$ is defined, for $a \in \cA$ and $(p,d_p) \in \bigcup_p \cA_p$ as
\begin{equation*}
    \MGS{\CCE}[a,(p,d_p)] = u_p(d_p,a_{-p})-u_p(a_p,a_{-p})
\end{equation*}
\end{definition}

The goal of Player $1$ is to minimize $\MGS{\CCE}[a,(p,d_p)]$. An existence of $\CCE$ in the original game implies that there exists a distribution $\mu^*$ over strategy profiles in that game, such that no deviation is profitable for any of the players. In particular, if Player $1$ plays $\mu^*$, this guarantees that the utility against any action of Player $2$ is at most $0$, which implies that the value of the $\CCE$ game, if exists, is at most $0$. 
Similarly, the mixed-strategy $\mu$ played by Player 1 in any $\epsilon$-approximate Nash equilibrium for the $\CCE$ game, constitutes an $\epsilon$-approximate CCE in the original game. Therefore, we would like to apply Algorithm~\ref{alg:nash} to find a Nash equilibrium of the $\CCE$ game. Yet, this requires two things: (1) Bounding the various dimensions of the $\CCE$ matrix in terms of those of the original game, which is proved using closure properties of these dimensions that appear in Appendix~\ref{app:concat}; and (2) Implementing best-response oracles for the players of the $\CCE$ game, based on best-response and value oracles for the original game. We notice that since the utility of Player 2 corresponds to the deviations of players from a strategy-profile given by Player 1, the best deviation can be simulated using a best response oracle for the original game. For Player 1, we will not simulate a best-response oracle. Rather, recall that such an oracle is used only for computing an approximate Nash equilibrium for the half-infinite game, where Player 2 is restricted to play from finitely many actions and Player 1 is unrestricted. Hence, it is sufficient to compute such a Nash equilibrium, a task that can be reduced to a computation of a CCE in a finite subgame of the original game.\footnote{We note that computing a $\CCE$ does not quite yield a Nash equilibrium in the $\CCE$-subgame, yet, guarantees to find a mixed strategy of value $0$ for Player 1 which is sufficient for the purpose of the algorithm.} In that finite game, the set of actions available to the players correspond to the actions of Player 1 in the $\CCE$ game.

\section*{Acknowledgements}
Most of this work was completed during an internship at Archimedes AI Research Center. The authors would like the thank them for their kind support.
Angelos Assos, Yuval Dagan, Costis Daskalakis, and Maxwell Fishelson are supported by NSF Awards CCF-1901292, DMS-2022448 and DMS2134108, a Simons Investigator Award, the Simons Collaboration on the Theory of Algorithmic Fairness and a DSTA grant.
Idan Attias is supported by the Vatat Scholarship from the Israeli Council for Higher Education and by the Kreitman School of Advanced Graduate Studies.
Yuval Dagan gratefully acknowledges the NSF's support of FODSI through grant DMS-2023505.

\bibliography{refs}

\newpage
\appendix

\section{Preliminaries}\label{sec:prelim}

\subsection{Function Classes and Dimensions}

We work with two types of function classes: 0-1 function classes, and the more general real-valued function classes.  We define them here as well as their accompanying notions of dimensionality.

\subsubsection{0-1 Function classes}

\begin{definition}[0-1 function class]
    Define a \bfc~ $\cF$ to be a set of \emph{concepts}, $f \colon X\to \{0,1\}$, where $X$ is called a \emph{domain set} and $\{0,1\}$ is a \emph{label set}.
\end{definition}

The three important notions of dimensionality of 0-1 function classes are the following.

\begin{definition}[VC Dimension]
    For a \bfc~ $\cF$, denote its VC Dimension $\VC(\cF)$ to be the maximal $d$ (possibly infinite) such that there exists a magnitude-$d$ subset $\lrset{x_1,\cdots,x_d}\subseteq X$ satisfying the following.  For all binary strings $b \in \lrset{0,1}^d$, there exists $f \in \cF$ that satisfies
    \begin{equation*}
        f(x_i) = b_j \qquad \text{for all } j \in [d]
    \end{equation*}
\end{definition}

\begin{definition}[Littlestone Dimension] \label{def:littlestone}
    For a \bfc~ $\cF$, denote its Littlestone Dimension $\Lit(\cF)$ to be the maximum depth $d$ of a complete, binary tree $T=(V,E)$, such that the children of any internal node are ordered by \emph{left} and \emph{right} and such that any internal node $v\in V$ is labeled by $x(v) \in X$ and every leaf is labeled by $f_v\in \cF$, and these labeling functions satisfy the following: let $v_1-v_2 \cdots -v_d-v_{d+1}=\ell$ be a root-to-leaf path along the tree, where $v_1$ is the root and $v_{d+1}$ is a leaf. Then, for any $i \in [d]$:
    \begin{align*}
        f_\ell(x(v_i)) &= 1 \qquad \text{if $v_{i+1}$ is a \emph{left} child of $v_i$}\\
        f_\ell(x(v_i)) &= 0 \qquad \text{if $v_{i+1}$ is a \emph{right} child of $v_i$}
    \end{align*}
\end{definition}

We also utilize the concept of threshold dimension, a concept similar to VC dimension but only requiring the existence of a hypothesis for each ``threshold binary string'' of the form $(0,\cdots,0,1,\cdots,1)$.

\begin{definition}[Threshold Dimension] \label{def: threshold dimension}
    For a \bfc~ $\cF$, denote its Threshold Dimension $\tr(\cF)$ to be the maximal $d$ (possibly infinite) such that there exist magnitude-$d$ subsets $\lrset{f_1,\dots,f_d}\subseteq \cF$ and $\lrset{x_1,\cdots,x_d}\subseteq X$ satisfying the following.
    \begin{align*}
        f_i(x_j) &= 1 \qquad \text{for all }i\leq j \in [d]\\
        f_i(x_j) &= 0 \qquad \text{for all }i > j \in [d]
    \end{align*}
\end{definition}
These dimensions are related in the following ways. 

\begin{lemma}\label{lem:dim-relate}
    $\VC(\cF) \leq \min(\Lit(\cF),\tr(\cF))$
\end{lemma}
The Lemma follows from standard arguments.
\begin{lemma}\label{lem:dim-relate-2}
   $\log \Lit(\cF) \leq \tr(\cF) \leq 2^{\Lit(\cF)}$
\end{lemma}
For the proof, see \cite[Theorem 3]{alon2019private} and the references therein \cite{hodges1997shorter,shelah1990classification}

\subsubsection{Real-valued function classes}

The following definitions are analogous to those for \bfc es.

\begin{definition}[Real-valued function class]
    Define a \rfc~ $\cF$ as a set of concepts $f \colon X \to \hR$. Similarly, define a $[0,1]$-valued function class as a collection of functions $f \colon X \to [0,1]$.
\end{definition}

When defining the three analogous concepts of dimensionality in the real-valued setting, we introduce a margin parameter $\epsilon$ to ensure hypotheses are sufficiently distinct.  Let us start with the real-valued analogue of VC dimension: ``\efsd''.

\begin{definition}[\efsd]
    For a \rfc~ $\cF$, denote its \efsd~ $\fat(\cF,\epsilon)$ to be the maximal $d$ (possibly infinite) such that there exists a magnitude-$d$ subset $\lrset{x_1,\cdots,x_d}\subseteq X$ and witnesses $(\theta_1,\cdots,\theta_d) \in \hR^d$ satisfying the following.  For all binary strings $b \in \lrset{0,1}^d$, there exists $f_b \in \cF$ with
    \begin{equation} \label{eq:fat-shattered}
    \begin{aligned}
        f_b(x_j) &\geq \theta_j+\epsilon &\qquad \text{for all }j \in [d] \text{ with }b_j=1\\
        f_b(x_j) &\leq \theta_j &\qquad \text{for all }j \in [d] \text{ with }b_j=0
    \end{aligned}
    \end{equation}
\end{definition}

Next, the real-valued analogue of Littlestone dimension: ``\esfsd''.

\begin{definition}[\esfsd] \label{def:sequential fat shattering}
    For a \rfc~ $\cF$, denote its \esfsd~ $\sfat(\cF,\epsilon)$ to be the maximum depth of a complete binary tree $T=(V,E)$, whose internal nodes $v\in V$ are labeled by elements $x(v) \in X$ and are accompanied by some \emph{witnesses} $\theta(v)$, whose leaves $\ell\in V$ are labeled by $f_\ell\in \cF$, such that the following holds: for any root-to-leaf path $v_1,\dots,v_d,v_{d+1}=\ell$ in the tree and for any $i \in [d]$:
    \begin{equation}\label{eq:sfsd}
    \begin{aligned}
        f_\ell(x(v_i)) &\geq \theta(v_i)+\epsilon & \qquad \text{if $v_{i+1}$ is a \emph{left} child of $v_i$}\\
        f_\ell(x(v_i)) &\leq \theta(v_i) & \qquad \text{if $v_{i+1}$ is a \emph{right} child of $v_i$}
    \end{aligned}
    \end{equation}
\end{definition}

For the real-valued analogue of threshold dimension (\eftd), we slightly shift our use of the witness parameters.  Rather than have a distinct witness $\theta_j$ for each $x_j$ (as in the definition of \efsd) or for each $v$ (as in the definition of \esfsd), we use the same $\theta$ across. The reasoning behind this stems from our eventual use of representing game matrices as function classes, where it is preferred to have a definition that is symmetric to transposing the matrix, or equivalently, to swapping the roles of the concept class $\cF$ and the domain set $X$. When we are considering bounded \rfc es (for example, only taking on values in $[0,1]$), defining the \eftd~ in this way will only lead to a $O(1/\epsilon)$ factor difference in the dimension.

\begin{definition}[\eftd] \label{def:eftd}
    For a \rfc~ $\cF$, denote its \eftd~ $\fatr(\cF,\epsilon)$ to be the maximal $d$ (possibly infinite) such that there exist magnitude-$d$ subsets $\lrset{f_1,\dots,f_d}\subseteq \cF$ and $\lrset{x_1,\dots,x_d}\subseteq X$ and a witness $\theta \in \hR$ satisfying the following.
    \begin{equation}\label{eq:fatrd}
    \begin{aligned} 
        f_i(x_j) &\geq \theta + \epsilon \qquad &\text{for all }i\leq j \in [d]\\
        f_i(x_j) &\leq \theta \qquad &\text{for all }i > j \in [d]
    \end{aligned}
    \end{equation}
\end{definition}

For \rfc es bounded on $[0,1]$, these dimensions are related in the following ways.
\begin{lemma}\label{lem:ineq-comb}
For any $[0,1]$-valued function class $\cF$ and any $\epsilon >0$,
\begin{equation} \label{eq:fat-equations}
    \fat(\cF,\epsilon) \leq \min\l(\sfat(\cF,\epsilon),\frac{\fatr(\cF,\epsilon/2)}{2\epsilon}\r).
\end{equation}
Further, there exist universal constants $c,C>0$ such that:
\begin{equation}\label{eq:sfat-tr}
    \frac{c\epsilon \log(\epsilon \log(\sfat(\cF,\epsilon))/C)}{\log(1/(C\epsilon))} \leq \fatr(\cF,\epsilon) \leq 2^{\sfat(\cF,\epsilon)+1}
\end{equation}
\end{lemma}

We note that the left part of Eq.~\eqref{eq:sfat-tr} was proved by \cite[Lemma~8.4]{daskalakis2022fast} and the remaining proofs are standard. We include sketches in Appendix~\ref{app:inequalities}.

Lastly, we define the convex hull of a function class:
\begin{definition}\label{def:convhull}
	The \emph{convex hull} of a function-class $\cF$ over domain $X$ is defined as:
	\[
	\conv(\cF) = \l\{\sum_{j=1}^m \lambda_j f_j \colon m \in \mathbb{N},~ f_1,\dots,f_m \in \cF,~ \lambda_1,\dots,\lambda_m \ge 0,~ \sum_{i=1}^m \lambda_i = 1 \r\}
	\]
    The \emph{dual convex hull} of a class is defined as a class on the domain $\conv(X)$ as the set of all formal convex combinations of finitely many elements from $X$, namely,
    \[
    \conv(X) = \lrset{\sum_{i=1}^\ell \lambda_i x_i \colon \ell \in \mathbb{N},~ \lambda_1,\dots,\lambda_\ell\in [0,1],~ \sum_{i=1}^\ell \lambda_i = 1}
    \]
    where $\sum_i \lambda_i x_i$ is a formal sum. Extend each $f \in \cF$ to the domain $\conv(X)$ by defining the extended function $\dconv(f)$ as 
    \[\dconv(f)\lr{\sum_i \lambda_i x_i} = \sum_i \lambda_i f(x_i) ~.\]
    Define the dual convex hull of $\cF$ as
    \[
    \dconv(\cF) = \lrset{\dconv(f) \colon f \in \cF}
    \]
    Lastly, define $\convtwo(\cF) = \dconv(\conv(\cF))$.
\end{definition}

\subsubsection{Uniform convergence}

Below, we use the notion of uniform convergence, which enables one to sample-down a distribution and compress it to a distribution over a small number of elements, while changing the expectation of each function by at most $\epsilon$.

\begin{definition}[Uniform convergence]
\label{def:uniform-conv}
For a concept class $\cF$ over a domain set $X$, define by $\c(\mathcal{F},\epsilon,\delta)$ the smallest number $m$, such that for any measure $\mu$ over $X$,
\[
\Pr_{x_1,\dots,x_m \sim \mu ~\text{(i.i.d)}}\l[
\forall f \in \cF,~ \l|\frac{1}{m} \sum_{i=1}^m f(x_i) - \E_{x\sim \mu}[f(x)] \r| \le \epsilon \r] \ge 1-\delta. 
\]
Denote by $\c(\cF,\epsilon) = \inf_{\delta < 1} \c(\cF,\epsilon,\delta)$.
\end{definition}

Intuitively, the support of any distribution can be compressed down to a size of at most $\c(\cF,\epsilon)$, while changing expectations of functions in $\cF$ by at most $\epsilon$.

We notice that $\c(\cF,\epsilon,\delta)$ can be bounded in terms of the VC dimension of 0-1 valued classes and in terms of the shattering numbers of real-valued classes: \cite{rudelson2006combinatorics}.
\begin{lemma} \label{lem:vershynin}
	Let $\cF$ be a concept class. If $\cF$ is 0-1 valued then
	\[
	c(\cF,\epsilon,\delta)
	\le C\frac{\VC(\cF) + \log(1/\delta)}{\epsilon^2},
	\]
	where $C>0$ is a universal constant. Similarly, if $\cF$ is $[0,M]$-valued then
	\[
	\c(\mathcal{F},\epsilon,\delta) \le C\frac{\l(\int_0^M \sqrt{\fat(\mathcal{F},\epsilon)} \r)^2 + \log(1/\delta)}{\epsilon^2} 
	\]
\end{lemma}
\subsection{Online learning and  Multiplicative-Weights algorithm}\label{app:online-MW}
We address the online learning setting in realizable and agnostic settings. 
Let $\cH\subseteq \lrset{0,1}^\cX$ be a hypothesis class, where $\cX$ is the instance space and $\cY=\lrset{0,1}$ is the label space. 
The online learning protocol can be formulated as a game between the learner and an adversary, where at rounds $t=1,2,\ldots,T$,
\begin{enumerate}
    \item The adversary chooses $(x_t,y_t)\in \cX\times\lrset{0,1}$.
    \item The learner observes $x_t$ and predicts $\hat{y}_t\in\lrset{0,1}$.
    \item The learner observes $y_t$ and suffers a loss $\lrabs{\hat{y}_t-y_t}$.
\end{enumerate}
In the section, we assume realizability which means that all target labeled are generated by a function $h^\star\in\cH$, that is, $h^\star(x_t)=y_t$ for $t\in [T]$.
Define the mistake bound for a deterministic algorithm $\cA:\lr{\cX\times\cY}^*\times \cX \rightarrow \lrset{0,1}$ by
\begin{align*}
    \cM(\cH,\cA)
    =
    \sup_{h^\star\in \cH, T, x_{1:T}}
    \sum^T_{t=1}\lrabs{\cA\lr{(x_{1:(t-1)},y_{1:(t-1)}),x_t}
    -h^\star(x_t)
    }.
\end{align*}
We can write the output of the algorithm as a function $\hat{c}_t(x)= \cA\lr{(x_{1:(t-1)},y_{1:(t-1)}),x}$.
We say that the algorithm is \emph{improper} if the output functions $\hat{c}_t$ do not belong to $\cH$.

The learner is allowed to make randomized predictions, where the adversary picks $(x_t,y_y)$ without knowing the random bits of the learner in this round. We analyze the expected loss of the learner. 
Formally, for a randomized algorithm $\cA:\lr{\cX\times\cY}^*\times \cX \rightarrow \Delta \lr{\lrset{0,1}}$, 
we define the expected loss of algorithm $\cA$ in $T$ rounds
\begin{align*}
    \cL(\cH,\cA,T)
    =
    \sup_{h^\star\in \cH, x_{1:T}}
    \sum^T_{t=1}\hE\lrabs{\cA\lr{(x_{1:(t-1)},y_{1:(t-1)}),x_t}-h^\star(x_t)},
\end{align*}
where the expected loss at round $t$ can be interpreted as the probability of predicting incorrectly at round $t$.
If we write the algorithm's (random) output as a function $\hat{c}_t(x)= \cA\lr{(x_{1:(t-1)},y_{1:(t-1)}),x}$,
we say that the algorithm is \emph{randomized proper} for a function class $\cH$ if the function $\hat{c}_t$ belong to $\cH$. In other words, the algorithm draws a function from a distribution that is supported on $\cH$.

In the agnostic setting, we define the more general setting of prediction with expert advice, where the loss function is arbitrary and not necessarily the 0-1 loss as in the realizable setting.

    \begin{definition}[Prediction with expert advice]
    There are $N$ experts indexed by $[N] = \{1,\dots,N\}$. 
In each time step $t=1,\ldots,T$ the learner chooses a probability vector $\mu^t=(\mu^t_1,\dots,\mu^t_N)$ from the simplex $\cS_N = \Lrset{p \in \hR^N : \forall i, \; p_i\geq 0 \; \text{and} \; \sum_{i=1}^N p_i=1 }$.
Thereafter, a loss vector $\ell_t \in [0,1]^N$ is revealed to the learner
In the adversarial setting, the loss vectors $\ell_1,\ldots,\ell_T$ are entirely arbitrary and may be chosen by an adversary.
The goal of the learner is to minimize the regret, given by 
\begin{align*}
      \regret_T
    :=
    \sum_{t=1}^T \mu^t \cdot \ell_t - \min_{i \in [N]} \sum_{t=1}^T \ell_{t}(i).  
\end{align*}
\begin{lemma}[Multiplicative weights algorithm]\label{lem:multiplicative}
Multiplicative Weights algorithm with learning rate $\eta=\sqrt{\frac{\log N}{2T}}$ suffers a regret of $O(\sqrt{T \log N})$ where $N$ is the number of experts and $T$ is the horizon length. Further, the time complexity is $O(N)$ per iteration.
\end{lemma}
    \end{definition}
For proof, see \cite{littlestone_learning_1988}, \cite{arora2012multiplicative},  
and \cite[Section 6]{schapire2013boosting}
    \begin{algorithm}[H]\label{alg:MW}
    \caption{\texttt{Multiplicative Weights}}\label{algMW}
    \textbf{Parameters:} Learning rate parameter $\eta>0$.
    \begin{enumerate}
        \item Initialize a uniform probability distribution $\mu^1=\l(\mu_1^1,\dots,\mu_{N}^1\r)$ over the $N$ experts, where $\mu_i^1 \gets \frac{1}{N}$ for all $i=1,\dots,N$.
        \item For $t=1,2,\ldots,T$:
        \begin{enumerate}
            \item Predict $\mu^t$ and suffer a loss $\mu^t\cdot\ell_t$.
            \item Update
                $\mu^{t+1}$ based on $\ell_t$: 
                \begin{align*}
                \forall i,~
                   \mu^{t+1}_i \gets \mu^{t}_i \exp\lr{- \eta \ell_t(i)}/Z^{t} \text{ where } Z^{t} = \sum_{j=1}^n \mu^{t}_j \exp\lr{- \eta \ell_t(j)}.
                \end{align*}
        \end{enumerate}
    \end{enumerate}
\end{algorithm}

\subsection{Games}

\subsubsection{Games and equilibria}
\begin{definition}[Multi-player game]
A {\em $k$-player game} is a pair $(\actions, u)$, where $\actions=\mathcal{A}_1 \times \cdots \times \mathcal{A}_k$ and $ u=(u_1,\dots,u_k)$, where each $u_p \colon \actions \to \hR$. We assume that each $\cA_i$ is accompanied with some $\Sigma$-algebra that $ u$ is measurable with respect to. Each $\mathcal{A}_p$ is the set of {\em actions} available to player~$p$ and each $u_p$ is the {\em utility, or payoff, function} of player $p$, which maps the set of {\em action profiles} $\cal A$ to the reals. Each player's goal is to maximize their own utility. We denote by $\cA_{-p}$ the Cartesian product~of~$\{\mathcal{A}_{j}\}_{q\ne p}$. Similarly, for any action $ a = (a_1,\dots,a_k)\in \actions$, denote by $a_{-p}$ the Cartesian product of $\{a_q\}_{q \ne p}$.
\end{definition}

\begin{definition}[Zero-sum game]
A \emph{zero-sum} game is a two-player game such that $u_1(a,b)=-u_2(a,b)$ for all $a\in \cA$ and $b\in \cB$. We sometimes compress our notation and represent a zero-sum game as $({\cal A},{\cal B},u)$ where $u:{\cal A}\times {\cal B} \rightarrow \hR$ is a single function representing the utility function of player~$2$. Player $2$ aims to maximize this utility while player $1$ aims to minimize this utility. 
\end{definition}

\begin{definition}[$\epsilon$-Nash equilibrium and $\epsilon$-CCE]
Let $(\actions,{u})$ denote a game. An $\epsilon$-approximate Nash equilibrium is a collection of probability measures, $\mu_1,\dots,\mu_k$, over $\cA_1,\dots,\cA_k$, respectively, such that for any player $p\in [k]$ and any $d_p \in \cA_p$,
\[
\E_{{ a} \sim \mu_1 \times \cdots \times \mu_k} [u_p(d_p,{ a}_{-p})]
\le 
\E_{ a  \sim \mu_1 \times \cdots \times \mu_k} [u_p( a)] - \epsilon.
\]
A \emph{Coarse Correlated Equilibrium} (CCE) is a joint measure $\mu$ over $\actions$ such that for any player $p$ and any $d_p \in \cA_p$,
\[
\E_{{ a} \sim \mu} [u_p(d_p,{ a}_{-p})]
\le 
\E_{ a  \sim \mu} [u_p( a)] - \epsilon.
\]
\end{definition}

\begin{definition}[Minimax equilibrium]
Given a zero-sum game $G=(\cA, \cB,u)$\footnote{We slightly abuse notation and throughout denote a zero-sum game $(\cA \times \cB,u)$ as $(\cA, \cB,u)$}, we say that $G$ has a \emph{minimax equilibrium} if
\begin{equation}\label{eq:minmax-app}
\inf_{\mu_1 \in \Delta(\cA)}\sup_{\mu_2\in \Delta(\cB)} \E_{a\sim \mu_1,b\sim\mu_2}[u(a,b)]
= 
\sup_{\mu_2\in \Delta(\cB)} \inf_{\mu_1 \in \Delta(\cA)} \E_{a\sim \mu_1,b\sim\mu_2}[u(a,b)]
\end{equation}
where $\Delta(\cA)$ and $\Delta(\cB)$ denote the set of all probability measures over $\cA$ and $\cB$ respectively. If Eq.~\eqref{eq:minmax-app} is satisfied, we say the probability measures $\mu_1,\mu_2$ optimizing Eq.~\eqref{eq:minmax-app} are a \emph{minimax equilibrium} of $G$ and denote by $\Val(G)$ \emph{value} of the game, which is the value of both sides of Eq.~\eqref{eq:minmax-app}.
\end{definition}

The following relation holds between an $\epsilon$-Nash and the value of the game:

\begin{lemma}\label{lem:nash-implies-val}
    Let $G(\cA,\cB,u)$ denote a zero-sum game with a minimax equilibrium. If $(\mu_1,\mu_2)$ is an $\epsilon$-approximate Nash equilibrium for this game, then
    \[
    \inf_{a \in \cA} \E_{b \sim \mu_2} [u(a,b)] \ge \Val(G)-\epsilon; \quad \sup_{b \in \cB} \E_{a \sim \mu_1}[u(a,b)] \le \Val(G) + \epsilon.
    \]
\end{lemma}

We note that a distribution over actions of a particular player $p$ is also termed a \emph{mixed strategy}, and given mixed strategies $\mu_1,\dots,\mu_k$ for the players, we abuse notation and denote for any $p \in [k]$ the utility of player $p$ given mixed strategies $\mu_1,\dots,\mu_k$ as $u_p[\mu_1,\dots,\mu_k] := \E_{{ a} \sim \mu_1 \times \ldots \times \mu_k}[u_p({ a})]$.

\subsubsection{Game dimensions}

We now extend these dimensionality definitions from function classes to games.  For each player $p$, her utility $u_p: \cA_p \times \cA_{-p} \to \hR$ can be thought of as a concept class $\cF_p$ over the domain set $X_p = \cA_p$, whose concepts $f_{a_{-p}}$ are parametereized by elements $a_{-p} \in \cA_{-p}$ and are defined by $f_{a_{-p}}(a_p) := u_p(a_p,a_{-p})$, for each $a_p \in \cA_p$. We define the dimension of a game to be the maximal dimension over these utility function classes, where $p$ ranges across all players.

\begin{definition}[Real-valued fat and threshold dimension for a multi-player game]\label{def:mat-form}
Let $G=(\actions,u)$ be a $k$-player game.  We define
\begin{align*}
    \fat(G,\epsilon) &= \max_p \fat(\cF_p,\epsilon)\\
    \fatr(G,\epsilon) &= \max_p \fatr(\cF_p,\epsilon)\\
    \sfat(G,\epsilon) &= \max_p \sfat(\cF_p,\epsilon) \\
    \c(G,\epsilon,\delta) &= \max_p \c(\cF_p,\epsilon,\delta)\\
    I(G) &= \lr{\int_0^1 \sqrt{\fat(G,\epsilon)}}^2
\end{align*}
\end{definition}

Lastly, the \emph{convex hull} of a game 
\[G=((\cA_1,\dots,\cA_k), u)\]
as the game $\conv(G) = ((\conv(\cA_1),\dots,\conv(\cA_k)),\tilde{u})$ where $\conv(\cA_p)$ is the set of all finitely-supported probability measures over $\cA_p$ and $\tilde{u}_p(\mu_1,\dots,\mu_k) := \E_{a_p \overset{\mathrm{i.i.d}}{\sim}\mu_p}[u_p(a_1,\dots,a_k)]$.

\subsubsection{Best-response oracle}

\begin{definition}[Best-response oracle]\label{def:bf}
Let $(\mathcal{A},u)$ be a multi-player game. An $\epsilon$-\emph{best-response} oracle receives a bounded-support distribution $\mu$ over $\mathcal{A}_{-p}$ and outputs an action $\hat a_p \in \cA_p$ that is an $\epsilon$-best response, namely:
\[
\E_{a_{-p} \sim \mu} u_p(\hat a_p,a_{-p})
\ge \sup_{a_p \in \cA_p} \E_{a_{-p} \sim \mu} u_p(a_p,a_{-p}) - \epsilon.
\]
\end{definition}

\subsection{On learnability and the existence of game equilibria} \label{sec:learnability-existence}

We say that a concept class $\cF$ is \emph{uniformly online learnable} if there is a function $R \colon \mathbb{N}\to [0,\infty)$, that satisfies $R(T) = o(T)$ as $t\to\infty$, such that for any $T$ there exists an online learner that achieves a regret of at most $R(T)$. 
We notice that a class is uniformly online learnable with regret of $o(T)$ if and only if its \esfsd~ is finite for all $\epsilon>0$. We notice that Algorithm~\ref{alg:online-learner-re} achieves a regret of $o(T)$ for any such class, using only a best response oracle, yet, with suboptimal regret.

Next, we discuss the reduction (see e.g., \cite[Section 6]{schapire2013boosting}) which shows that online learnability of an appropriate function-class implies the existence of an equilibria. 
First, we define the notion of a \emph{repeated game}: this denotes the iterative setting, where in each iteration $i=1,\dots,T$, each of the players is playing an action and gains a reward according to their utility. Equivalently, the negation of their reward can be viewed as a loss that they suffer. The next lemma states that if both players in a zero-sum game play a no-regret learning algorithm, then the average-iterate converges to a Nash equilibrium. Similarly, in a general sum game the average iterates converge to a CCE:
\begin{lemma}[Equilibria computation via a repeated game] \label{lem:equi-via-online}
Assume a repeated game between two players in a zero-sum game, that is repeated for $T$ iterations, where in each iteration $t$, player $1$ plays an action $a_1^t$ and player $2$ plays $a_2^t$. Assume that each player $p$ suffers a regret of $\epsilon T$, for $p \in\{1,2\}$, namely:
\[
\sum_{t=1}^T u_p(a_p^t,a_{-p}^t) \ge 
\sup_{a_p \in \cA_p} \sum_{t=1}^T u_p(a_p,a_{-p}^t) - \epsilon T.
\]
Denote by $\mu_p$ the uniform distribution over $a_p^1,\dots,a_p^T$, for $p\in\{1,2\}$. Then, $(\mu_1,\mu_2)$ is an $\epsilon$-Nash equilibrium.

Similarly, in a general-sum multi-player game, assume a repeated game such that every player $p \in [k]$ plays action $a_p^t$ in iteration $t\in [T]$ and suffers a regret of at most $\epsilon T$. Denote by $\mu$ the joint distribution over $\actions$ which is the uniform distribution over the multiset $\{(a_1^t,\dots,a_k^t)\}_{t\in[T]}$. Then, $\mu$ is an $\epsilon$-$\CCE$.
\end{lemma}

In particular, this implies that any game which admits no-regret learners has a minimax (for a zero-sum game) or a $\CCE$ (for general games). This implies that games with bounded \esfsd~ for all $\epsilon>0$ admit such equilibria. Equivalently, games with bounded \eftd~ $\epsilon>0$ attain such equilibria.

The converse is not completely true, yet, for zero-sum games, there are known lower bounds that use similar notions of a dimension. For example, in a 
0-1 valued game, assume that the game contains a subgame which is an infinitely large threshold game, namely, there exist actions $\{a_1^n\}_{n\in \mathbb{N}}$ and $\{a_2^n\}_{n\in \mathbb{N}}$ such that $u(a_1^n,a_2^m) = 1$ if $m\ge n$ and $0$ otherwise. Then, as observed by \cite{hanneke2021online}, this subgame does not contain a Nash Equilibrium. There is a gap between the above described upper and lower bound, which is the setting where the threshold dimension is infinite, namely, there are arbitrarily large threshold games, yet, there is no infinitely large threshold subgame. In this particular setting, \cite{hanneke2021online} showed that if additional the VC dimension is finite then the game admits a Nash Equilibrium.

\section{Deferred proofs for Online Learning}\label{app:online learning}

\begin{lemma}\label{lem:optimize-alpha}
    Let $0<\epsilon\le 1/2$, let $d \ge 1$, let $\alpha=\epsilon/(d+1)$ and let $J$ such that
    \[
    1 \le J \le \lr{\frac{1}{\epsilon-\alpha}}^d.
    \]
    Then,
    \[
    \frac{J \log J}{\alpha^2} \le C d^3 \log(1/\epsilon) \lr{\frac{1}{\epsilon}}^d
    \]
    where $C>0$ is a universal constant.
\end{lemma}
\begin{proof}
First,
\begin{align*}
\frac{1}{\epsilon-\alpha}
= 
\frac{1}{\epsilon}\frac{1}{1-1/(d+1)}
= 
\frac{1}{\epsilon}\lr{1+\frac{1}{d}}~.  
\end{align*}
Consequently, 
\begin{equation*}
\begin{aligned}
\frac{J\log J}{\alpha^2}
\le 
\frac{d\log(1/(\epsilon-\alpha))}{\alpha^2}\lr{\frac{1}{\epsilon-\alpha}}^{d}
&=
d^3\lr{\log(1/\epsilon) + \log(1+1/d)}\lr{\frac{1}{\epsilon}}^{d} \lr{1+\frac{1}{d}}^{d} 
\\
&\le 
C d^3 \log(1/\epsilon) \lr{\frac{1}{\epsilon}}^{d}.
\end{aligned}
\end{equation*}
\end{proof}

\begin{lemma} \label{lem:improp-lit}
    The improper variant of Algorithm~\ref{alg:online-learner-re} makes at most $C^{\Lit(\cF)}$ mistakes, for a universal constant $C>0$.
\end{lemma}
\begin{proof}
Substitute $\epsilon=1/2$ and $\alpha=1/4$.
It has argued, in the main proof body, that $\fatr(\dconv(\cH\oplus h^*),\epsilon-\alpha) \ge J$, under the assumption that the algorithm runs for more than $J$ phases. 
By Theorem~\ref{thm:conv-bnd},
\begin{align*}
J &\le \dconv(\cH\oplus h^*,\epsilon-\alpha) = e^{C \sfat(\cH\oplus h^*,(\epsilon-\alpha)/C)/(\epsilon-\alpha)^2}
= e^{C \Lit(\cH\oplus h^*)/(\epsilon-\alpha)^2} \\
&= e^{C' \Lit(\cH\oplus h^*)}
\end{align*}
for universal constants $C,C'>0$.
It is well known and follows from definition that $\Lit(\cH)=\Lit(\cH\oplus h^*)$, which yields that
\[
J \le e^{C \Lit(\cH)}~.
\]
Since the number of times the distribution $\mu^t$ of the algorithm makes an $\epsilon$-mistakes in each round $j$ is bounded by $\cO\lr{\log(j)/\alpha^2} \le \cO\lr{\log(J)}$ (recall $\alpha=1/4$), the total number of times of such an $\epsilon$-mistake across all phases is bounded, up to constants, by
\[
J\log J
\le 
C \Lit(\cH) e^{C \Lit(\cH)}~,
\]
where $C$ is a universal constant. Recall that in the improper setting, the algorithm predicts according to the majority. Consequently, it makes a mistake if $\mu^t$ makes an $\epsilon=1/2$ mistake. Hence, the number of mistakes of the algorithm is bounded by $\cO(e^{C\Lit(\cH)})$.
\end{proof}

\subsection{Agnostic Online Learning}
\begin{proof}[Proof of Theorem~\ref{thm:agnostic}]
We start with the bound for improper learners.
We use the reduction of \cite{ben2009agnostic} to reduce from agnostic to realizable. In their proof, they instantiate $\binom{T}{M}$ algorithms for the realizable setting, where $M$ is the mistake bound of the realizable algorithm (i.e. the regret bound). Each of these instantiations is fed with the original sequence of $x_1,\dots,x_T$, however, the labels $y_1,\dots,y_T$ are different in each instantiation. A multiplicative weights algorithm (Algorithm~\ref{algMW}) is used to choose between these experts. In their paper, they proved that the regret of this algorithm is bounded by $\cO(\sqrt{TM\log T})$\footnote{In their paper, \cite{ben2009agnostic} provided an algorithm for the agnostic setting with a mistake bound of $O(\sqrt{T\Lit(\cH)\log T})$, however, the only property they used of the Littlestone's dimension is the existence of an algorithm with a mistake bound $\Lit(\cH)$ for the realizable setting. Hence, we can replace $\Lit(\cH)$ with the mistake bound of any algorithm and the proof would follow.}, which translates to the bound of Theorem~\ref{thm:agnostic}, when one substitutes $M$ with mistake bound (i.e. the regret bound) of the improper learner from Theorem~\ref{thm:online-realizable}. In the randomized proper setting, instead of a bound on the number of mistakes, the proof of Theorem~\ref{thm:online-realizable} yields a bound $M(\epsilon)$ on the number of $\epsilon$-mistakes. While translating it into the reduction of \cite{ben2009agnostic}, one obtains a total regret of $\cO\lr{\sqrt{TM(\epsilon) \log T} + T \epsilon}$, where $T\epsilon$ accounts for the additional loss caused by the fact that each of the experts possibly suffers an additional regret of $T\epsilon$, accounting for less-that-$\epsilon$-mistakes. Substituting the bound of $M(\epsilon) \le \widetilde\cO\lr{\epsilon^{-2\tr(\cH)-2}}$ from Theorem~\ref{thm:online-realizable} and substituting $\epsilon=T^{-1/(2\tr(\cH)+4)}$ yields a total regret of $T^{\frac{2\tr(\cH)+3}{2\tr(\cH)+4}}$.

\paragraph{Number of oracle calls and Runtime:} 
Since the algorithm selects from $\binom{T}{M}$ using multiplicative weight update, the runtime is linear in the number of experts and equals $\cO\lr{\binom{T}{M}}$. In the improper setting, $M = \widetilde{\cO}\lr{4^{\tr(\cH)}}$, which translates to a bound of $\binom{T}{M} \le T^{\widetilde{O}\lr{4^{\tr(\cH)}}}$ and in the proper setting, one a bound on $M(\epsilon) \le \widetilde\cO\lr{\epsilon^{-2\tr(\cH)-2}}$, which yields a regret of $T^{\widetilde\cO\lr{\epsilon^{-2\tr(\cH)-2}}}$. For the specific choice of $\epsilon=T^{-1/(2\tr(\cH)+4)}$, the runtime per iteration is at most $T^{\widetilde{\cO}\lr{T^{(2\tr(\cH)+3)/(2\tr(\cH)+4)}}}$.

The number of oracle calls can be trivially bounded by the runtime. Though, one can obtain a better bound. There are at most $t^d$ distinct functions if we restrict our hypothesis class to a domain of size $t$. In particular, there are at most $t^d$ maximal realizable partial labelings of $\{x_1,\dots,x_t\}$. It suffices to call the consistent oracle only on these maximal labelings (and the collection maximal realizable labelings of $x_1,\dots,x_{t+1}$ can be efficiently constructed from those over $x_1,\dots,x_t$). This yields at bound of at most $T^d$ oracle calls per iteration. 
\end{proof}

\section{Minmax in zero-sum games} \label{sec:minimax-app}

This section is dedicated to the proof of Theorem~\ref{thm:zero sum eq computation}.
We denote the game as $G = (\cA,\cB,u)$, where $\cA$ and $\cB$ are the sets of actions of Players $1$ and $2$, respectively. In Section~\ref{sec:alg-half-inf} we will describe an algorithm that solves the minimax of a game where one player has infinitely many actions at hand, and she is equipped with a best-response oracle, and the other player has finitely many actions, to which he has random access. Next, in Section~\ref{sec:alg-nash} we describe an algorithm, that iteratively uses the algorithm for the half-infinite game to compute a Nash equilibrium for a game where both players have infinitely many actions.

\subsection{Nash for a half-infinite game}\label{sec:alg-half-inf}

In this section, we will describe an algorithm to find a Nash equilibrium for a game where one of the players has a finite set of actions and the other player has an infinite set of actions. Assume without loss of generality that player $1$ has the finite action set. A similar approach to computing an equilibria in half-infinite games appeared in the context of robust PAC learning \citep{feige2015learning,attias2022improved}.

In order to compute the Nash equilibrium, recall from Lemma~\ref{lem:equi-via-online} that if the players play a repeated game and if both players play a no-regret learning algorithm, then their average-iterate converge to a Nash equilibrium. In particular, multiplicative weights (Algorithm~\ref{alg:MW}) and $\epsilon$-best response are both no-regret learning algorithms, therefore, the following lemma is a (well known) consequence of Lemma~\ref{lem:multiplicative}:
\begin{lemma}[Multiplicative weights vs. best response]  \label{lem:exp-vs-best-response}
    Assume a repeated zero-sum game for $T$ iterations, between player $1$ who plays over $n$ actions and player $2$ whose number of actions is unbounded and possibly infinite. Assume that Player $1$ plays the exponential weights algorithm to choose a mixed strategy $\mu^t$ over their actions in each iteration $t \in [T]$, and assume that player $2$ reacts with an $\epsilon$-best-response $b_t$ to the uniform mixture of $\mu^1,\dots,\mu^t$, namely,
    \[
    \frac{1}{t}\sum_{i=1}^t u(\mu^i,b_t) \ge \sup_{b \in \cB} \frac{1}{t}\sum_{i=1}^t u(\mu^i,b) - \epsilon.
    \]
    Denote by $\mu$ the uniform mixture of $\mu^1,\dots,\mu^T$ and by $\xi$ the uniform distribution over $b_1,\dots,b_T$. Then, $(\mu,\xi)$ is an $O\l(\sqrt{\log(n)/T}+\epsilon\r)$-Nash equilibrium.
\end{lemma}
Following Lemma~\ref{lem:exp-vs-best-response}, we propose Algorithm~\ref{alg:nash-half-inf}, where Player $1$ plays exponential weights over her actions and Player $2$ plays a best-response, using his best-response oracle.
\begin{algorithm}[H]
\caption{\texttt{$\epsilon$-Nash Equilibrium for a half-infinite zero-sum game}}
\label{alg:nash-half-inf}
\textbf{Input:} A game $(\cA = \{a_1,\dots,a_n\},\cB,u)$, a value $\epsilon>0$.\\
\textbf{Subroutines:}
\begin{itemize}
    \item \texttt{BestResponse} oracle: receives a mixed strategy over actions from $\cA$ and an $\epsilon>0$ and outputs an $\epsilon$-best response from $\cB$ (see Definition~\ref{def:bf}).
    \item The utility function $u(a,b)$ that receives a pair of actions and outputs its utility.
\end{itemize}
\begin{enumerate}
    \item $T \gets \l\lceil \frac{C \log |\cA|}{\epsilon^2} \r\rceil$ ~; $\eta \gets \sqrt{\frac{\log\lrabs{\cA}}{2T}}$.
    \item $\mu^1=\l(\mu_1^1,\dots,\mu_{n}^1\r)$ denotes a uniform probability distribution over $\cA$ where $\mu_i^1 \gets \frac{1}{n}$ for all $i=1,\dots,n$.
    \item $b_1 \gets \texttt{BestResponse}(\mu^{1},\epsilon)$.
    \item \textbf{For}  $t=2,\dots,T$
    \begin{enumerate}
        \item \textbf{For} $i=1,\dots,n$, 
        \[
        \mu^{t}_i \gets \mu^{t-1}_i \exp(\eta u(a_{t-1},b_i))/Z^{t} \text{ where } Z^{t} = \sum_{j=1}^n \mu^{t-1}_j \exp(\eta u(a_{t-1},b_j))
        \]
        \item $b_t \gets \texttt{BestResponse}(\mu^t,\epsilon)$
    \end{enumerate}
    \item \textbf{Return} $(\bar\mu,\bar\xi)$, where $\bar\mu$ is the uniform mixture of $\mu^1,\dots,\mu^t$, namely, $\bar\mu_i = \frac{1}{T} \sum_{t=1}^T \mu^t_i$ and $\bar\xi$ is the uniform distribution over $b_1,\dots,b_T$
\end{enumerate}
\end{algorithm}
We obtain the following statement:
\begin{lemma}[Nash for the half-infinite game]\label{lem:oracle-reduction}
Let $G=(\cA=\{a_1,\dots,a_n\},\cB,u)$ be a zero-sum game where $|\cA|=n$ and $\cB$ is possibly infinite, and let $\epsilon>0$. Then, Algorithm~\ref{alg:nash-half-inf}, executed with the parameter $\epsilon$, finds an $O(\epsilon)$-Nash equilibrium, after $T = O(\log n/\epsilon^2)$ iterations.
\end{lemma}

\begin{proof}
Algorithm~\ref{alg:nash-half-inf} implements exponential-weights vs best response, therefore, the proof follows directly from Lemma~\ref{lem:exp-vs-best-response}.
\end{proof}

\subsection{The algorithm for the fully-infinite game}\label{sec:alg-nash}

We first argue that the output of the algorithm is an $O(\epsilon)$-approximate Nash equilibrium.
\begin{lemma}\label{lem:out-is-minmax}
Assume that  Algorithm~\ref{alg:nash} stops. Then, the returned strategies constitute a $5\epsilon$-Nash for the original game $(\cA,\cB, u)$.
\end{lemma}
\begin{proof}
Since the game ends, then either $\Val(A_t,B_{t-1}) \ge \Val(A_{t-1},B_{t-1}) - \epsilon$ or $\Val(A_t,B_{t}) \le \Val(A_t,B_{t-1}) + \epsilon$. Assume first that the latter holds.

Notice that for any $t$,
\begin{equation}\label{eq:progress-A}
\Val(A_t,B_{t-1}) \le \Val(\mathcal{A},B_{t-1}) + \epsilon.
\end{equation}
This is due to the fact that in the game $(A_t,B_{t-1},u)$, player 1 has a strategy that guarantees her a value of at least $\Val(\cA,B_{t-1})$: indeed, her strategy $\xi^{t,1}$ satisfies this property, because $\xi^{t,1}$ is a strategy for player 1 in the game $(\cA,B_{t-1},\epsilon)$. Similarly,
\begin{equation}\label{eq:progress-B}
\Val(A_t,B_t) \ge \Val(A_t,\cB) - \epsilon.
\end{equation}
Recall that we assumed that $\Val(A_t,B_t) \le \Val(A_t,B_{t-1}) + \epsilon$. This with the equations above yields:
\[
\Val(A_t,\cB)-\epsilon
\le \Val(A_t,B_t)
\le \Val(A_t,B_{t-1}) + \epsilon
\le \Val(\cA,B_{t-1}) + 2\epsilon.
\]
Recall that $\mu^{t,2}$ is the strategy for player $2$ in an $\epsilon$-Nash in the game $(\cA,B_{t-1},u)$. This implies that for any $a\in \cA$,
\begin{equation} \label{eq:eps-nash-is-good}
u(a,\mu^{t,2}) \ge \Val(\cA,B_{t-1}) - \epsilon
\ge \Val(A_t,B_t)-3\epsilon.
\end{equation}
Similarly, since $\xi^{t,1}$ is the strategy for player $1$ in an $\epsilon$-Nash for the game $(A_t,\cB,u)$, then for every $b \in \cB$:
\[
u(\xi^{t,1},b) \le \Val(A_t,\cB) + \epsilon \le \Val(A_t,B_t) + 2\epsilon.
\]
Combining the equations above, we obtain that
\[
\Val(A_t,B_t) - 3\epsilon \le u(\xi^{t,1},\mu^{t,2})
\le \Val(A_t,B_t)+2\epsilon.
\]
Consequently, 
\[
u(\xi^{t,1},\mu^{t,2}) - \inf_{a\in \cA} u(a,\mu^{t,2}) 
\le \Val(A_t,B_t)+2\epsilon - \l(\Val(A_t,B_t)-3\epsilon\r) = 5\epsilon,
\]
and similarly,
\[
\sup_{b\in \cB} u(\xi^{t,1},b)
- u(\xi^{t,1},\mu^{t,2}) \le \Val(A_t,B_t)+2\epsilon - \l(\Val(A_t,B_t)-3\epsilon\r) = 5\epsilon.
\]
This concludes that $(\xi^{t,1}\mu^{t,2})$ is a $5\epsilon$-Nash, and recall that we assumed that $\Val(A_t,B_{t}) \le \Val(A_t,B_{t-1}) + \epsilon$. This was one of the stopping conditions. However, we have to consider the second stopping condition, namely, that $\Val(A_t,B_t) \ge \Val(A_{t-1},B_{t}) - \epsilon$. Recall that Eq.~\eqref{eq:progress-A} states that 
\[\Val(A_t,B_{t-1}) \le \Val(\mathcal{A},B_{t-1}) + \epsilon\] 
and we substitute $t-1$ instead of $t$ in Eq.~\eqref{eq:progress-B} to obtain that \[\Val(A_{t-1},B_{t-1}) \ge \Val(A_{t-1},\cB) - \epsilon.\]
Combining with our assumption that 
\[\Val(A_t,B_{t-1}) \ge \Val(A_{t-1},B_{t-1}) - \epsilon,\] 
we obtain that
\[
\Val(A_{t-1},\cB)-2\epsilon
\le \Val(A_{t-1},B_{t-1}) - \epsilon
\le \Val(A_t,B_{t-1})
\le \Val(\cA,B_{t-1}) + \epsilon.
\]
Since $\mu^{t,2}$ is the strategy of player $2$ in an $\epsilon$-Nash for the game $(\cA,B_{t-1},u)$, we obtain that for any $a \in \cA$,
\[
u(a,\mu^{t,2})
\ge \Val(\cA,B_{t-1}) - \epsilon 
\ge \Val(A_{t-1},B_{t-1}) -3\epsilon.
\]
Further, since $\xi^{t,1}$ is the strategy of player 1 in the $\epsilon$-Nash for the game $(A_t,\cB,u)$ and since the Nash equilibrium is monotone under addition of actions to a single player, for any $b \in \cB$,
\[
u(\xi^{t,1},b)
\le \Val(A_t,\cB) + \epsilon
\le \Val(A_{t-1},\cB) + \epsilon
\le \Val(A_{t-1},B_{t-1}) + 2\epsilon.
\]
The proof concludes similarly to how it ended in the first of the two cases that we analyze.
\end{proof}

Next, we are bounding the stopping time of the algorithm. 

\begin{lemma}\label{lem:stopping-time-alg}
Algorithm \ref{alg:nash} stops after $O(\fatr(\conv(G),\epsilon)/\epsilon)$ iterations.
\end{lemma}

\begin{proof}[of lemma \ref{lem:stopping-time-alg}]
Suppose the game runs for more than $T$ iterations. First, we would like to find indices $1 \leq i_1 <  i_2 < \dots < i_q < T$ for which the following hold:
\begin{equation}\label{eq:crossing}
\begin{aligned}
\Val(A_{i_j}, B_{i_k}) \le \theta & \qquad \text{ if } j>k \\
\Val(A_{i_j}, B_{i_k}) \ge \theta + \epsilon/2 & \qquad \text{ if } j\leq k.
\end{aligned}
\end{equation}
We claim that we can always guarantee $q = O(T\epsilon)$ such indices.

Define a 'crossing' of an interval $[a,b]$ where $0 \leq a < b \leq 1$ to be a pair of numbers $c,d$ such that $c \leq a < b \leq d$. We will prove that we can find $q$ pairs of the form $(\Val(A_{i}, B_{i-1}),\Val(A_i, B_i))$ that cross the same interval. Indeed, if we take the numbers $\Val(A_i, B_{i-1})$, for $1 \leq i \leq T$, by Lemma~\ref{lem:unique-threshold} we have that at least $\epsilon T/2$ of them lying in the interval $U = \l[\theta-\frac{\epsilon}{2}, \theta\r]$, for some $\theta$, and denote them by $i_1 < \cdots < i_q$. Since the algorithm has not stopped at any of these iterations, by the stopping condition of the algorithm it holds that $\Val(A_{i_j},B_{i_j}) \ge \Val(A_{i_j},B_{i_j-1}) + \epsilon \ge \theta +\epsilon/2$. This concludes that $(\Val(A_{i_j}, B_{i_j-1}),\Val(A_{i_j}, B_{i_j}))$ crosses the interval $[\theta,\theta + \epsilon/2]$. This concludes Eq.~\eqref{eq:crossing}: indeed, for all $j > k$, since $i_j > i_k$, consequently, $i_k \le i_j-1$ then, by the definition of the algorithm, $B_{i_k} \subseteq B_{i_j-1}$, which implies, by the monotonicity of the value of the game, that 
\[
\Val(A_{i_j},B_{i_k})
\le \Val(A_{i_j}, B_{i_j-1})
\le \theta.
\]
Similarly, one can deduce that $\Val(A_{i_j},B_{i_k}) \ge \theta+\epsilon/2$ for $j \le k$, which concludes Eq.~\eqref{eq:crossing}.

Recall that taking the average of $\c(G, \frac{\epsilon}{12},\delta)$ (defined in Definition \ref{def:mat-form}) actions can guarantee us an  $\frac{\epsilon}{12}$-approximate minmax strategy  $\alpha_{i_j}$ for player $1$ in the game $(A_{i_j},B_{i_{j-1}},u)$, due to uniform convergence (\ref{def:uniform-conv}). Similarly we can get a  $\frac{\epsilon}{12}$-approximate minmax strategy $\beta_{i_j}$ for player $2$ in the game $(A_{i_j},B_{i_j},u)$. For these strategies, notice that for all $ j>k$, it holds that
\[
u(A_{i_j},B_{i_k}) \le \theta + \frac{\epsilon}{6}
\]
and for all $j \le k$, it holds that
\[
u(A_{i_j},B_{i_k}) \ge \theta + \frac{\epsilon}{3}.
\]
Now, look at the matrix where the rows are parameterized by the strategies of player $1$, $\{\alpha_{i_j}\colon 1 \leq j \leq q\}$, the columns by the strategies of player $2$, $\{\beta_{i_k}\colon 1 \leq k \leq q\}$, and the value of the $(j,k)$ entry is taken to be $u(\alpha_{i_j},\beta_{i_k})$. This constitutes a matrix where all entries above and at the diagonal are at least $\theta+\frac{\epsilon}{3}$ and below the diagonal at most $\theta+ \frac{\epsilon}{6}$. This concludes that the $\frac{\epsilon}{6}$-fat threshold dimension of the mixed-strategy game is at least $\Omega(T\epsilon)$, namely $\fatr(\conv(G),\epsilon) \ge \Omega(T\epsilon)$ which gives $T \leq O(\frac{\fatr(\conv(G),\epsilon)}{\epsilon})$, as required.
\end{proof}
In order to bound the number of oracle calls, we add the following lemma:
\begin{lemma}\label{lem:bnd-oracle-calls}
    Assume that Algorithm~\ref{alg:nash} runs for $T$ iterations. Then, the number of oracle calls is bounded by $\cO(T/\epsilon^2 \cdot \log(T/\epsilon^2))$.
\end{lemma}
\begin{proof}
First, we would like to bound the sizes of $A_t$ and $B_t$ by $\cO(t \log t/\epsilon^2)$. In order to show that, notice that $A_{t}$ is obtained from $A_{t-1}$ by adding the support of an $\epsilon$-approximate Nash for the half-infinite game $(\cA, B_{t-1}, u)$. We would like to bound the size of the support of the strategy of Player 1 in such an approximate Nash. This approximate Nash is computed in Algorithm~\ref{alg:nash-half-inf}, and the size of the support equals the number of iterations of this algorithm, which is bounded by $\cO(\log |B_{t-1}|/\epsilon^2)$, by Lemma~\ref{lem:oracle-reduction}. This implies that $|A_t| \le |A_{t-1}| + C (\log |B_{t-1}|+C)/\epsilon^2 $ for a universal constant $C>0$ and similarly, since $B_t$ is obtained from $B_{t-1}$ by adding the support of an approximate Nash for the game $(A_t,\cB,u)$, it holds that $|B_t| \le |B_{t-1}| + C (\log |A_t|+C)/\epsilon^2$. It can be proven by induction that $|A_t|,|B_t| \le \cO((t/\epsilon^2)\log(t/\epsilon^2))$. We notice that the number of oracle calls by iteration $t$ equals exactly $|A_t|+|B_t|$, which concludes the proof.
\end{proof}

We are now ready to prove our main theorem.

\noindent \textbf{Theorem~\ref{thm:zero sum eq computation}}
\textit{Let $G = (\cA,\cB,u)$ be a zero-sum two-player game, where $u \colon \cA\times \cB \to [0,1]$ and let $\epsilon>0$. There is an algorithm to find an $O(\epsilon)$-Nash for this game using the following number of $\epsilon$-best response oracle calls (assuming this number is finite):
\[
\cO\lr{\min\lr{e^{C\sfat(G,\epsilon/C)/\epsilon^2},~ (1/\epsilon)^{CI(G)^2 \fatr(G,\epsilon/C)/\epsilon^5}}} ;~ I(G) = \lr{\int_0^1 \sqrt{\fat(G,\epsilon)}}^2
\]
}
\begin{proof}[Proof of Theorem~\ref{thm:zero sum eq computation}]
We notice that Lemma~\ref{lem:out-is-minmax} implies that the output of Algorithm \ref{alg:nash} is an $O(\epsilon)$-Nash. Furthermore, Lemma~\ref{lem:stopping-time-alg} bounds the number of iterations by $T \le \cO\lr{\fatr(\conv(G),\epsilon)/\epsilon}$ and Lemma~\ref{lem:bnd-oracle-calls} implies that the number of oracle calls is bounded by $\widetilde\cO\lr{\fatr(\conv(G),\epsilon)/\epsilon^3}$. The proof of Theorem~\ref{thm:zero sum eq computation} follows by substituting $\fatr(\conv(G),\epsilon)$ with its bound in terms of the various dimensions of the game $G$, according to Theorem~\ref{thm:conv-bnd}:
\[
\fatr(\conv(G),\epsilon) \le e^{C\sfat(G,\epsilon/C)/\epsilon^2}\enspace.
\]
The total bound on the number of oracle calls is then
\[
\widetilde\cO\lr{\fatr(\conv(G),\epsilon)/\epsilon^3} \le 
e^{C\sfat(G,\epsilon/C)/\epsilon^2}
/\epsilon^3.
\]
Notice that the fact of $1/\epsilon^3$ can be omitted by changing the constant in the exponent, and this yields the desired bound on the number of oracle calls and concludes Theorem~\ref{thm:zero sum eq computation}.
    
\end{proof}

\section{Bounding the threshold dimension of the mixed-strategy game}\label{sec:thresh-dim-mixed-proof}

In this Sections~\ref{sec:using-ramsey}~and~\ref{sec:bnd-fat} we prove Theorem~\ref{thm:conv-bnd}, which consists of two upper bounds on $\fatr(\conv(\cF),\epsilon)$ and $\fatr(\conv(\cG),\epsilon)$ for a zero-sum game $G$ and a concept class $\cF$. We further present lower bounds in Section~\ref{sec:conv-lb}.

\subsection{A bound using Ramsey's theory} \label{sec:using-ramsey}

In this section, we will prove the following lemma:

\begin{lemma}\label{lem:mixed-vs-pure-threshold}
Let $G$ be a zero-sum game, then, 
\[
\fatr(\conv(G),\epsilon) \le O\left( \left(\frac{2}{\epsilon}\right)^{\c(G,\epsilon/4)^2\fatr(G,\epsilon/4)/\epsilon}\right).
\]
\end{lemma}
We notice that substituting $\c(G,\epsilon/4)$ with $\lr{\int_{0}^1 \sqrt{\fat(G,\epsilon)}d\epsilon}^2/\epsilon^2$ using Lemma~\ref{lem:vershynin}, yields one of the two bounds of Theorem~\ref{thm:conv-bnd}.

In the proof, we will use the following variants of Ramsey numbers, as defined below:

\begin{definition}[Multi-colored Ramsey number] 
\begin{itemize}
    \item $R(n,Q)$ is defined as the maximal size of a complete graph that contains no monochromatic $n$-clique, where each edge can be colored by one of $Q$ colors.
    \item $R(n,Q,\ell)$ is defined as the maximal size of a complete graph that contains no $n$-monochromatic clique, if each edge is colored by $\ell$ colors out of $Q$ possible colors.
    \item $R\Big((n_1, n_2, \dots, n_Q),Q,\ell\Big)$ is defined as the maximal size of a complete graph that contains no $n_i$-monochromatic clique of color $i$, if each edge is colored by $\ell$ colors out of $Q$ possible colors. Note that $R\Big((n, n, \dots, n),Q,\ell\Big) = R(n,Q,\ell)$
\end{itemize}
\end{definition}

The following bound holds:

\begin{proposition}[A bound on $R(m,Q)$, from \cite{balaji2021pigenhole} Corollary 3.9]
For $n,Q \in \mathbb{N}$, $R(n,Q)$ can be bounded as follows
$$R(n,Q) \leq \frac{3 + e}{2} \frac{(Q(n-2))!}{((n-2)!)^Q}$$
Using the bound $m! = O(m^m)$, we can write the above as:
$$R(n,Q) \leq O(Q^{Qn})$$
\end{proposition}

Here, we show how to improve the bound if we know that each edge is colored with $\ell$ colors:
\begin{proposition}[A bound on $R(n,Q,\ell)$]
\label{prop:Ramsey}
For $n,Q,l \in \mathbb{N}, l < Q$ it holds that:
$$R\Big(n, Q, \ell \Big) \leq 2 \Big(\frac{Q}{\ell} \Big)^{(n-2)Q}$$
\end{proposition}
An immediate corollary from the above proposition is the following:
\begin{corollary}[A bound on $R(n,Q,\epsilon Q)$]
    \label{Ramsey-corollary}
    We have that for $n,Q \in \mathbb{N}$ and $\epsilon \in (0,1)$
    $$R(n,Q,\epsilon Q) \leq O\left( \left(\frac{1}{\epsilon}\right)^{Qn}\right)$$
\end{corollary}
Now to prove Proposition \ref{prop:Ramsey}, we use the following:

\begin{proposition}[Proposition 2.1, \cite{balaji2021pigenhole}]
Suppose $\ell, Q, n_1, \dots , n_r \in \mathbb{N}$,        and suppose $A_1, \dots , A_Q$
are sets with $|A_1 \cup \dots \cup A_Q| = n$. If $n > n_1 + \dots + n_Q - Q$, then $|A_i| \geq n_i$ for some $1 \leq i \leq Q$
\end{proposition}
\begin{lemma}
    \label{Ramsey:Lemma}
    For some $Q,\ell, n_1, n_2, .., n_Q \in \mathbb{N}$, we have that:
    \begin{align*}
    &\ell \left ( R\Big((n_1, n_2, \dots, n_r), Q, \ell \Big) -1 \right) \\
    &\qquad\leq R\Big((n_1-1, n_2, \dots, n_Q), Q, \ell \Big) + \dots + R\Big((n_1, n_2, \dots, n_Q-1), Q, \ell \Big) - (Q - 1)
    \end{align*}
\end{lemma}

\begin{proof}[of Lemma \ref{Ramsey:Lemma}]
    Assume $N = R\Big((n_1, n_2, \dots, n_r), Q, \ell \Big)$. Let $v$ be an arbitrary node out of the $N$ nodes of the $K_N$ graph and denote is the \emph{parent node}. Allocate the $N-1$ remaining vertices into $r$ sets, namely $A_1, A_2, ..., A_Q$, where each node can appear in multiple sets. A node $u$ will be assigned to set $A_c$, where $c= 1,2, \dots, Q$ if one of the colors of the edge that connects $u$ and $v$ is $c$. Now, note that
    $$|A_1 \cup A_2 \cup ... \cup A_Q| = (N-1)\ell$$
    since each edge has $\ell$ colors. If we assume towards contradiction that: 
    $$(N-1)\ell > R\Big((n_1-1, n_2, \dots, n_r), Q, \ell \Big)+ \dots + R\Big((n_1, n_2, \dots, n_r-1), Q, \ell \Big) - (Q - 1)$$
    then there exists an $i$ such that $|A_i| \geq R\Big((n_1, \dots, n_i-1 ,\dots n_r), Q, \ell \Big)$. If such an $i$ exists, that means that we either have a $K_{n_j}$ clique for $i \not = j$ in $A_i$, or we have a $K_{n_i - 1}$ clique in $A_i$, and therefore we can create a $K_{n_i}$ clique by connecting the parent node $v$ with the $n_i - 1$ nodes in $A_i$ that form that clique. That means that we have a color $i$ for which we have a monochromatic clique of size $n_i$, contradicting the definition of the multi-colored Ramsey number $R\Big((n_1, n_2, \dots, n_r), Q, \ell \Big)$. Thus it must be that
    $$(N-1)\ell > R\Big((n_1-1, n_2, \dots, n_r), Q, \ell \Big)+ \dots + R\Big((n_1, n_2, \dots, n_r-1), Q, \ell \Big) - (Q - 1)$$
    as required.
\end{proof}

\begin{proof}[of Proposition \ref{prop:Ramsey}]
    Let us define the following quantity
    $$S(N, Q, \ell) = \max_{(n_1, n_2, \dots, n_Q) | n_i \geq 2, \sum_i n_i = N} R\Big( (n_1, \dots, n_Q), Q, \ell \Big)$$
Note that $S(2Q, Q, \ell) = 2$, as we have that $\sum_i^Q n_i = 2Q$ which means either $n_i = 2 \forall i$ in which case the multi-colored Ramsey number is 2, or there exists $i$ such that $n_i=1$ in which case the number is 1. We can use Lemma \ref{Ramsey:Lemma} as follows:
\begin{align*}
S(N, Q, \ell) &= R\Big( (n_1, \dots, n_Q), Q, \ell \Big) \\
& \leq \frac{R\Big( (n_1-1, \dots, n_Q), Q, \ell \Big) + \dots + R\Big( (n_1, \dots, n_Q-1), Q, \ell \Big)}{\ell} + 1 - \frac{Q-1}{\ell} \\ 
 & \leq \frac{R\Big( (n_1-1, \dots, n_Q), Q, \ell \Big) + \dots + R\Big( (n_1, \dots, n_Q-1), Q, \ell \Big)}{\ell} \leq \frac{Q}{\ell} S(N-1, Q, \ell) 
\end{align*}
Using the above we can conclude that:
$$S(N, Q, \ell) \leq \Big(\frac{Q}{\ell}\Big)^{N - 2Q}S(2Q,Q,\ell) = 2\Big(\frac{Q}{\ell}\Big)^{N - 2Q}$$
Thus for the Multi colored Ramsey number $R\Big( (n_1, \dots, n_Q), Q, \ell \Big)$ we have 
$$R\Big(n, Q, \ell \Big) \leq S(nQ, Q, \ell) \leq 2 \Big(\frac{Q}{\ell} \Big)^{(n-2)Q}$$
\end{proof}
We proceed with the proof of Lemma~\ref{lem:mixed-vs-pure-threshold} as well as Corollary~\ref{cor:mixed-vs-pure-threshold-binary} for the binary setting.
Denote $m = \fatr(\conv(G),\epsilon)$ and $n=\fatr(G,\epsilon/2))$.
Since the $\epsilon$-fat threshold dimension of the mixed-strategy game is $m$, that means that there exist strategies $\alpha_{i}, \beta_{i},~ i \in [m]$ and a threshold $\theta \in [0,1]$ for which $u(\alpha_{i}, \beta_{j}) \le \theta$ if $i > j$ and $u(\alpha_{i}, \beta_{j}) \ge \theta + \epsilon$  if $i \leq j$. These strategies' supports may be unbounded. Towards applying the bound on the multi-colored Ramsey numbers, we would like to replace these strategies with strategies that have the above properties, with perhaps a gap smaller than $\epsilon$, however with bounded support. We use the uniform-convergence parameter of the game from Definition~\ref{def:mat-form}, $\c(G,\epsilon/8)$, to argue that for any $i \in [m]$, there exists a strategy $\alpha'_{i}$ for Player 1, that is a uniform distribution on $\c(G,\epsilon/8)$ actions (possibly with repetitions), which approximates the strategy $\alpha_{i}$ up to an error of $\epsilon/8$. More precisely, for any strategy $\beta$ of player $2$ and any $i \in [m]$,
\[
\l| u(\alpha'_{i}, \beta) - u(\alpha_{i}, \beta) \r| \le \frac{\epsilon}{8}
\]
Similarly, there exist strategies $\beta'_{i}, i \in [m]$ supported on $\c(G,\epsilon/8)$ actions, such that for any strategy $\alpha$ of player 1,
\[\l| u(\alpha, \beta'_{i}) - u(\alpha, \beta_{i}) \r| \le \frac{\epsilon}{8}~.\]
Thus, for strategies $\alpha'_{i}, \beta'_{i}$ we will have $u(\alpha'_{i}, \beta'_{j}) \le \theta + \frac{\epsilon}{4}$ if $i > j$ and $u(\alpha'_{i}, \beta'_{j}) \ge \theta + \frac{3\epsilon}{4}$  if $i \leq j$.

We want to bound the $\epsilon$-fat threshold dimension of the pure strategy game. To do so, let us take the graph $K_m$. Suppose $Q = \c(G,\epsilon/8)^2$. We determine how to color an edge $(i,j)$, where $i > j$, in the following manner: create two matrices of dimensions $Q \times Q$ each, denoted by $A^{i,j}$ and $B^{i,j}$. 
The matrix $A^{i,j}$ is the utility matrix of the subgame where Player 1 is restricted to play from $\mathrm{support}(\alpha'_i)$ and Player 2 from $\mathrm{support}(\beta'_j)$. In particular, denoting $\mathrm{support}(\alpha'_i) = \{a_{i,1},\dots,a_{i,\c(G,\epsilon/8)} \}$
and $\mathrm{support}(\beta'_j) = \{b_{j,1},\dots,b_{j,\c(G,\epsilon/8)} \}$, we define $A^{i,j}_{k,\ell} = u(a_{i,k},b_{j,\ell})$. Similarly, the matrix $B^{i,j}$ corresponds to the game where Player 1 plays from $\mathrm{support}(\alpha_j)$ and Player 2 from $\mathrm{support}(\beta_i)$, where $B^{i,j}_{k,\ell} = u(a_{j,k} b_{i,\ell})$.

Since $\alpha'_i$ and $\beta'_j$ are each uniform distributions over their supports, then for all $i > j$, $u(\alpha'_i,\beta'_j) = \frac{1}{Q^2}\sum_{k,\ell} A^{i,j}_{k,\ell}$ and similarly $u(\alpha'_j,\beta'_i) = \frac{1}{Q^2}\sum_{k,\ell} B^{i,j}_{k,\ell}$. Since for $i > j$, $u(\alpha'_i,\beta'_j) \le \theta + \frac{\epsilon}{4}$ and  $u(\alpha'_j,\beta'_i) \ge \theta + \frac{3\epsilon}{4}$, it holds that $\frac{1}{Q^2}\sum_{k,\ell} \l(B^{i,j}_{k,\ell}-A^{i,j}_{k,\ell}\r) \ge \frac{\epsilon}{2}$.

Suppose out of the $Q^2$ pairs $(k,l)$,  $\delta Q^2$ are such that $B^{i,j}_{k,\ell}-A^{i,j}_{k,\ell} \geq \frac{\epsilon}{4}$ and for the rest we have that $B^{i,j}_{k,\ell}-A^{i,j}_{k,\ell} < \frac{\epsilon}{4}$. That implies that
$$\frac{\epsilon}{2} \leq \frac{1}{Q^2}\sum_{k,\ell} \l(B^{i,j}_{k,\ell}-A^{i,j}_{k,\ell}\r) \leq \frac{(1-\delta)Q^2 \frac{\epsilon}{4} + \delta Q^2}{Q^2} =(1-\delta)\frac{\epsilon}{4} + \delta \implies \delta \geq \frac{\frac{\epsilon}{4}}{1-\frac{\epsilon}{4}}\geq \frac{\epsilon}{4} $$
Consider now all the pairs $(k,l)$ for which $B^{i,j}_{k,\ell}-A^{i,j}_{k,\ell} \geq \frac{\epsilon}{4}$ holds. Divide the interval of $[0,1]$ into $\lceil \frac{8}{\epsilon}\rceil$ intervals of size $\frac{\epsilon}{8}$ $\Big(I_1 = [0,\frac{\epsilon}{8}), I_2 = [\frac{\epsilon}{8},\frac{2\epsilon}{8}), \dots \Big)$ and assign the pair $(k,l)$ in the interval which $A_{k,l}^{i,j}$ belongs to. If the pair belongs to interval $I_r$, we color the edge $(i,j)$ with color $(k,l,r)$. Note that in total there are $Q^2 \lceil \frac{8}{\epsilon}\rceil$ colors. 

Suppose we can form a monochromatic clique of size $t$ in the graph $K_m$ we constructed, with vertices $u_1, u_2, \dots, u_t$. That means that there exist actions $a_1, a_2, \dots, a_t \in \cA$ and $b_1, b_2, \dots, b_t \in \cB$ such that $u(a_i, b_j) \leq \theta'$ if $i > j$ and $u(a_i, b_j) \geq \theta' + \frac{\epsilon}{8}$ if $i \leq j$. That implies that the $\epsilon$-fat threshold dimension of the game is at least $t$. Above, we constructed a complete graph, where each edge is multicolored with $\frac{\epsilon Q^2}{4}$ colors out of the available $Q^2\lceil \frac{8}{\epsilon}\rceil$ colors. By Corollary~\ref{Ramsey-corollary}, we have that a monochromatic clique in $K_m$ of size $n+1$ will exist if $m \geq O\left( \left(\frac{1}{\epsilon}\right)^{\frac{1}{\epsilon}Q^2n}\right) \geq R(n + 1, Q^2\lceil \frac{8}{\epsilon}\rceil, \frac{\epsilon Q^2}{4})$. As we assumed the $\epsilon-$fat threshold dimension of the pure strategy game is exactly $n$, we need to have $m \leq O\left( \left(\frac{1}{\epsilon}\right)^{\frac{1}{\epsilon}Q^2n}\right) = O\left( \left(\frac{1}{\epsilon}\right)^{\frac{1}{\epsilon}\c(G,\epsilon/8)^2n}\right)$ as otherwise we would have a threshold of at least $n+1$.

With a slight modification to the previous argument, we attain the following corollary for binary-valued games,

\begin{corollary}\label{cor:mixed-vs-pure-threshold-binary}
Let $G$ be a $\lrset{0,1}$-valued zero-sum game, then, 
\[
\tr(\conv(G),\epsilon) \le O\left( (1/\epsilon)^{C\VC(G)^2\tr(G)\log(\VC(G))}\right)
\] 
\end{corollary}

\begin{proof}
    As in the general case of real valued games, for all $i>j \in [m]$, we can construct the $Q \times Q$ matrix $A^{i,j}$ with entries $A^{i,j}_{k,l} = u(a_{i,k},b_{j,\ell})$ for $a_{i,k} \in \mathrm{support}(\alpha'_i)$ and $b_{j,\ell} \in \mathrm{support}(\beta'_j)$ and the $Q \times Q$ matrix $B^{i,j}$ with entries $B^{i,j}_{k,l} = u(a_{j,k},b_{i,\ell})$ for $a_{j,k} \in \mathrm{support}(\alpha'_j)$ and $b_{i,\ell} \in \mathrm{support}(\beta'_i)$.  Again, we have $\frac{1}{Q^2}\sum_{k,\ell} \l(B^{i,j}_{k,\ell}-A^{i,j}_{k,\ell}\r) \ge \frac{\epsilon}{2}$.
    
    Now, importantly though, the matrices have values in $\lrset{0,1}$. This means there must exist at least 1 pair $(k,l)$ for each $(i,j)$ with $B^{i,j}_{k,\ell}-A^{i,j}_{k,\ell} = 1$.  So, rather than discretizing the $[0,1]$ interval and coloring each edge $i,j$ using $Q^2 \left\lceil \frac{8}{\epsilon} \right \rceil$ colors, we can use simply $Q^2$ colors, and a monochromatic clique will imply a threshold matrix in the pure strategy game.  Even though each $(i,j)$ pair receives one color here, the reduction in number of colors gives us the improved $m \leq O\left( \left(Q^2\right)^{Q^2n}\right)$, as desired.

\end{proof}

\subsection{A bound based on the sequential fat-shattering dimension of the original game} \label{sec:bnd-fat}

In this section, we prove one of the two bounds from Theorem~\ref{thm:conv-bnd}, stated as the following lemma:
\begin{lemma}\label{lem:using-conv}
	Let $\cF$ be a $[0,1]$-valued concept class. Then, there exist universal constants $C_1,C_2>0$ such that
	\[
	\fatr(\conv(\cF),\epsilon)
	\le 2^{C_1\sfat(\cF,\epsilon/C_2)/\epsilon^2} \enspace.
	\]
    Further, the same holds when $\conv(\cF)$ is replaced with either $\dconv(\cF)$ and $\convtwo(\cF)$, and when $\cF$ is replaced with a zero-sum game $G$.
\end{lemma}

We use the definition of the sequential Rademacher complexity of a class $\cF$ \citep{rakhlin2010online}:
\begin{definition}
Given a concept class $\cF$ over a domain $X$, given horizon length $T>0$, and given functions $Z_0,\dots,Z_{T-1}$ where $Z_t \colon \{-1,1\}^t \to X$, define the \emph{sequential Rademacher complexity of $\cF$} with respect to $Z=(Z_0,\dots,Z_{T-1})$ as
	\[
	\sRad(\cF,T,Z) = \E_{\epsilon_1,\dots,\epsilon_T}\l[\sup_{f \in \cF} \sum_{t=1}^T \epsilon_t f(Z_{t-1}(\epsilon_1,\dots,\epsilon_{t-1})) \r],
	\]
 where $\epsilon_1,\dots,\epsilon_T$ are sampled uniformly and independently from $\{-1,1\}$. Define $\sRad(\cF,T) = \sup_Z \sRad(\cF,T)$.
\end{definition}

The following are upper and lower bounds on the sequential Rademacher complexity given the sequential fat-shattering dimension were initially proved by \cite[Proposition~9]{rakhlin2010online} and the upper bound was improved by \cite{block2021majorizing}.
\begin{lemma}\label{lem:srad-sfat}
Let $\cF$ be a concept class, and let $\epsilon > 0$. Then, for any $T \ge \sfat(\cF,\epsilon)$, it holds that
\[
\sRad(\cF,T) \ge c\epsilon\sqrt{T \sfat(\cF,\epsilon)},
\]
where $c>0$ is a universal constant. Further, if $\cF$ is $[0,1]$-valued, then for any $T \ge 1$,
\[
\sRad(\cF,T) \le C\l(T\epsilon + \int_\epsilon^1 \sqrt{T \sfat(\cF, r)}dr \r),
\]
where $C>0$ is a universal constant.
\end{lemma}

Further, we use the following lemma \citep[Lemma~3]{rakhlin2010online}:
\begin{lemma}\label{lem:srad-conv}
	Let $\cF$ be a concept class. Then, $\sRad(\cF) = \sRad(\conv(\cF))$.
\end{lemma}

Consequently, we obtain the following bound:
\begin{lemma}
	Let $\cF$ be a $[0,1]$-valued concept class and let $\epsilon \in (0,1]$. Then,
	\[
	\sfat(\conv(\cF),\epsilon)
	\le \frac{C_1\sfat(\cF,\epsilon/C_2)}{\epsilon^2}
	\]
	where $C_1,C_2>0$ are universal constants.
\end{lemma}
\begin{proof}
	Let $T = \sfat(\conv(\cF),\epsilon)$, denote by $c,C$ the constants of Lemma~\ref{lem:srad-sfat}, let $\epsilon'>0$. Apply Lemma~\ref{lem:srad-sfat}, Lemma~\ref{lem:srad-conv} and the fact that $\sfat(\cF,r)$ is monotonic decreasing in $r$ to obtain:
	\begin{equation*}
		\begin{aligned}
	c \epsilon \sqrt{T \sfat(\conv(\cF),\epsilon)}
	\le \sRad(\conv(\cF),T)
	= \sRad(\cF,T) \\
	\le C\l(T\epsilon' + \int_{\epsilon'}^1 \sqrt{T \sfat(\cF, r)}dr \r)
	\le C T \epsilon' + C \sqrt{T \sfat(\cF,\epsilon')}\\
	= C\epsilon' \sqrt{T \sfat(\conv(\cF),\epsilon)} + C \sqrt{T \sfat(\cF,\epsilon')}.
	\end{aligned}
	\end{equation*}
	Setting $\epsilon'=c/(2C)$, one obtains that
	\[
	\frac{c\epsilon}{2} \sqrt{T \sfat(\conv(\cF),\epsilon)}
	\le C \sqrt{T \sfat(\cF,\epsilon')}.
	\]
	Consequently,
	\[
	\sfat(\conv(\cF),\epsilon)
	\le \l( \frac{2C}{c}\r)^2 \frac{\sfat(\cF,\epsilon')}{\epsilon^2}.
	\]
	This concludes the proof.
\end{proof}

We use that $\fatr(\cF,\epsilon) \le 2^{\lceil \sfat(\cF,\epsilon)\rceil}$ to obtain the desired bound. In order to obtain the bounds where $\conv(\cF)$ is replaced with either $\dconv(\cF)$ or $\convtwo(\cF)$, we notice that the proof follows from the same arguments, replacing Lemma~\ref{lem:srad-conv} with $\sRad(\cF,T) = \sRad(\conv(\cF),T) = \sRad(\dconv(\cF),T) = \sRad(\convtwo(\cF),T)$, whose proof follows similar arguments as the proof of Lemma~\ref{lem:srad-conv}. Lastly, for a game $G$, the bound follows directly from the bound for concept classes $\cF$, since the various dimensions of $G$ are obtained by considering the relevant concept classes.

\subsection{A lower bound}\label{sec:conv-lb}

We prove two lower bounds on $\fatr(\conv(\cF),\epsilon)$: one in terms of $\tr(\cF)$ and another one in terms of $\Lit(\cF)$, for some 0-1 classes $\cF$. This is stated in the following two lemmas.

\begin{lemma}
    For any $n \in \mathbb{N}$, $n \ge 2$ and any $\epsilon<1$, there exists a 0-1 valued concept class $\cF$ with $\fatr(\conv(\cF),\epsilon) = n$ whereas $\tr(\cF) \le \cO(\log n)$ (where $\cO$ hides constants that depend only on $\epsilon$).
\end{lemma}
\begin{proof}
We will prove for $\epsilon=0.2$ however the proof for any $\epsilon$ follows the same arguments with the numerical constants changed.
    We will define a distribution over concept classes and show that with probability greater than $0$, a random concept class from this distribution satisfies the desired properties. Define the random concept class $\cF$ as a union of classes $\cF_1,\dots,\cF_n$ over the domain $X = [n]$. Each $\cF_i$ contains $\ell$ different elements, $f_{i,1},\dots,f_{i,\ell}$ where $\ell = \Theta(\log n)$ is to be determined exactly later. We independently, for each $i \in [n]$, $k \in [\ell]$ and $j \in [n]$, define
    \[
    f_{i,k}(j) = \begin{cases}
        \text{$1$ w.p $0.7$ and $0$ w.p $0.3$} & \text{if } i \le j \\
        \text{$1$ w.p $0.3$ and $0$ w.p $0.7$} & \text{if } i > j.
    \end{cases}
    \]
    We notice that for each $i \le j \in [n]$, from Chernoff-Hoeffding bound,
    \[
    \Pr\lrbra{\frac{1}{k}\sum_{k=1}^\ell f_{i,k}(j) \ge 0.6} \ge 1 - e^{-c \ell},
    \]
    where $c>0$ is a universal constant, and similarly for any $i > j \in [n]$,
    \[
    \Pr\lrbra{\frac{1}{k}\sum_{k=1}^\ell f_{i,k}(j) \le 0.4} \ge 1 - e^{-c \ell}.
    \]
    We set $\ell = \log(3n^2)/c$ and notice that $e^{-c\ell} \le 1/(3n^2)$. By a union bound over $i,j \in [n]$, we get that with probability at least $2/3$, for all $i,j \in [n]$: 
    \begin{align*}
    \sum_{k=1}^\ell f_{i,k} \ge 0.6 & \qquad \text{if } i\le j \\ 
    \sum_{k=1}^\ell f_{i,k} \le 0.4 & \qquad \text{if } i> j 
    \end{align*}
    If this holds, then $\fatr(\conv(\cF),0.2) = n$. This is obtained by taking the convex combinations $\{\frac{1}{\ell} \sum_{k=1}^\ell f_{i,k} \}_{i\in[n]}$ as the functions $f_i$ in Definition~\ref{def:eftd}.

    Lastly, we will show that with probability at least $1/3$, $\tr(\cF) \le O(\log n)$. Let $M>0$ and we will see that if $M \ge \Omega(\log n)$ (with a sufficiently large constant), there are no functions $f^1,\dots,f^M \in \cF$ and elements $x_1,\dots,x_M \in X$ such that for all $i,j\in [M]$, $f^{i}(x_j) = \cI(i \le j)$. Fix some functions $f^1,\dots,f^M \in \cF$ and $x_1,\dots,x_M\in X$, and notice that 
    \[
    \Pr\lrbra{\forall i,j\in [n],~ f^{i}(x_j) = \cI(i \le j)}
    \le 0.7^{M^2},
    \]
    since $f^i(x_j)$ are independently chosen and each value for $f^i(x_j)$ can be taken with probability at most $0.7$ and there are $M^2$ values to be satisfied. We note that there are at most $(\ell n)^M n^M \le O(n^2\log n)^M$ choices for $\{f^1,\dots,f^M,x_1,\dots,x_M\}$, and by a union bound over all choices, we have
    \begin{align*}
    &\Pr[\tr(\cF) \ge M]
    \\
    &=\Pr\lrbra{\exists f^1,\dots,f^M \in \cF,~\exists x_1,\dots,x_M\in X,~ s.t.~ \forall i,j\in [M],~ f^i(x_j) = \cI(i \le j)}\\
    &\le (Cn^2\log n)^M 0.7^{M^2}
    \le e^{C_1 \log(n) M - C_2 M^2},
    \end{align*}
    where $C,C_1,C_2$ are universal constants. If we set $M = C_3 \log(n)$ for a sufficiently large constant $C_3$, then the probability above is bounded by $1/3$. We obtain that with probability $1/3$, $\fatr(\conv(\cF),0.2)=n$ and $\tr(\cF) \le O(\log(n))$ as required.
\end{proof}

Lastly, we notice the following lower bound $\fatr(\conv(\cF),\epsilon)$ in terms of Littlestone's dimension of the class:
\begin{lemma}
    For any $n$ there exists a 0-1 valued class $\cF$ such that $\Lit(\cF)\le\log n$ while $\fatr(\conv(\cF),\epsilon) = n$, for all $\epsilon <1$. Similarly, there a class $\cF$ such that $\Lit(\cF)\le\log n$ while $\fatr(\dconv(\cF),\epsilon) = n$.
\end{lemma}
\begin{proof}
    For the first part of the lemma that involves $\conv(\cF)$,
    this follows from the fact that $\fatr(\conv(\cF),\epsilon) \le \Lit(\cF)$, that $\Lit(\cF)\le \log |\cF|$, and that for $n$ there exists a class $\cF$ of $n$ elements with $\tr(\cF) = n$: Indeed, consider the class $\cF = \{f_1,\dots,f_n\}$ where $f_i \colon [n] \to \{0,1\}$ defined by $f_i(j) = \cI(i\le j)$. This class has threshold dimension $n$ by definition. The second part of the lemma, that involves $\dconv(\cF)$, is proved similarly.
\end{proof}

\section{CCE}\label{app:CCE}

Recall our definition of the $\CCE$-matrix of a game\\

\noindent \textbf{Definition \ref{def:CCEM} (The $\CCE$-matrix of a game)}
\textit{For a game $G=(\cA=\prod_{p=1}^k \cA_p,u=(u_1,\cdots,u_k))$, the \emph{$\CCE$-matrix} $\MGS{\CCE}: \cA \times \lr{\bigcup_p \cA_p}$ is defined, for $a \in \cA$ and $d_p \in \cA_p$ as}
\begin{equation*}
    \MGS{\CCE}[a,(p,d_p)] = u_p(d_p,a_{-p})-u_p(a_p,a_{-p})
\end{equation*}

As stated in Section \ref{sec:CCE-main}, in order to compute a $\CCE$ of $G$, we would like to run something analogous to Algorithm~\ref{alg:nash} on the zero-sum matrix game on $\MGS{\CCE}$.  To do so, we need two things:
\begin{enumerate}
    \item Bounds on the dimension parameters of $\MGS{\CCE}$ \label{it:app-dim-bound}
    \item Best-response oracles for the two players of the game \label{it:app-best-resp}
\end{enumerate}

For item \ref{it:app-dim-bound}, viewing $\MGS{\CCE}$ as a $[-1,1]$-concept class $\lrset{f_a|a \in \cA}$ over the domain set $X=\bigcup_p \cA_p$, we prove the following.

\begin{lemma}\label{lem:CCE-mat-dim-bound}
Let $G=(\cA,u)$ be a $k$-player game with bounded utilities $u_p: \cA \to [0,1]$ for all $p$.  Then, the combinatorial dimensions of the game $G$ bound those of the $\MGS{\CCE}$ concept class as follows
\begin{align*}
    \fat(\MGS{\CCE} ,16\epsilon)&\leq O(k\fat(G,\epsilon)\log\fat(G,\epsilon) \log^2(1/\epsilon))\\
    \sfat(\MGS{\CCE},2\epsilon)&\leq O((k/\epsilon)\sfat(G,\epsilon))\\
    \fatr(\MGS{\CCE},2\epsilon)&\leq (k/\epsilon)\fatr(G,\epsilon)\\
\end{align*}
\end{lemma}
\begin{proof}
We introduce the notation $\MGS{\CCE}[\cA,\cA_p]$ to denote the submatrix of $\MGS{\CCE}$ containing only the columns in $\cA_p$.  We can express $\MGS{\CCE}$ as the horizontal concatenation of matrices:  

\begin{equation*}
    \MGS{\CCE} = \left[\MGS{\CCE}[\cA,\cA_1], \cdots , \MGS{\CCE}[\cA,\cA_k]\right]
\end{equation*}

We will bound the dimensions of $\MGS{\CCE}[\cA,\cA_p]$ for all $p$ and use that to bound the dimensions of $\MGS{\CCE}$ using the horizontal concatenation lemmas of Appendix \ref{app:concat}.  We partition the rows of each matrix $\MGS{\CCE}[\cA,\cA_p]$ as follows.  For each integer $z \in [0,1/\epsilon]$, define $$S_{p,z,\epsilon} =\lrset{a \in \cA|u_p(a_p,a_{-p}) \in [z\epsilon,(z+1)\epsilon)}$$  We can express $\MGS{\CCE}[\cA,\cA_p]$ as the vertical concatenation of matrices: 

\begin{equation*}
    \MGS{\CCE}[\cA,\cA_p] = \begin{bmatrix}
    \MGS{\CCE}[S_{p,0,\epsilon},\cA_p]\\
    \MGS{\CCE}[S_{p,1,\epsilon},\cA_p]\\
    \vdots
    \end{bmatrix}
\end{equation*}

We will bound the dimensions of $\MGS{\CCE}[S_{p,z,\epsilon},\cA_p]$ for all $z$ and use that to bound the dimensions of $\MGS{\CCE}[\cA,\cA_p]$ using the vertical concatenation lemmas of Appendix \ref{app:concat}. Recall the utility concept class of player $p$. $\cF_p$ is defined over the domain set $X_p = \cA_p$, and has concepts $f_{a_{-p}}$ parametereized by elements $a_{-p} \in \cA_{-p}$ defined as $f_{a_{-p}}(d_p) := u_p(d_p,a_{-p})$, for each $d_p \in \cA_p$. Note, for all $(a_p,a_{-p}) \in S_{p,z,\epsilon}$ we have $u_p(a_p,a_{-p}) \in [z\epsilon,(z+1)\epsilon)$.  Therefore, for all $(a_p,a_{-p}) \in S_{p,z,\epsilon}$ and $d_p \in \cA_p$,
\begin{align*}
    \MGS{\CCE}[(a_p,a_{-p}),(p,d_p)] - (z+1/2)\epsilon &\in [f_{a_{-p}}(d_p)-\epsilon/2,f_{a_{-p}}(d_p)+\epsilon/2]
\end{align*}

Since the $(z+1/2)\epsilon$-shifted concepts of $\MGS{\CCE}[S_{p,z,\epsilon},\cA_p]$ correspond to concepts of $\cF_p$ up to an additive factor of $\epsilon/2$, any shattering structure present in $\MGS{\CCE}[S_{p,z,\epsilon},\cA_p]$ with a margin of $2\epsilon$ must exist in $\cF_p$ with a margin of at least $\epsilon$.  This is due to the fact that the $(z+1/2)\epsilon$ shift can be incorporated in the witness parameters of the structure $\theta$.
Thus, \begin{equation}\label{eq:CCE-dim-bound-1}
\begin{aligned}
    \fat(\MGS{V}[S_{p,z,\epsilon},\cA_p],2\epsilon)&=\fat(\cF_p,\epsilon)&\leq \fat(G,\epsilon)\\
    \sfat(\MGS{V}[S_{p,z,\epsilon},\cA_p],2\epsilon)&=\sfat(\cF_p,\epsilon)&\leq \sfat(G,\epsilon)\\
    \fatr(\MGS{V}[S_{p,z,\epsilon},\cA_p],2\epsilon)&=\fatr(\cF_p,\epsilon)&\leq \fatr(G,\epsilon)
\end{aligned}
\end{equation}

\noindent Using the vertical concatenation lemmas of Appendix \ref{app:concat}, as well as equation \eqref{eq:CCE-dim-bound-1}, we conclude
\begin{align}
    \fat(\MGS{\CCE} [\cA,\cA_p],16\epsilon)&\leq O(\fat(G,\epsilon)\log\fat(G,\epsilon) \log^2(1/\epsilon)) \label{eq:vfat}\\
    \sfat(\MGS{\CCE}[\cA,\cA_p],2\epsilon)&\leq O((1/\epsilon)\sfat(G,\epsilon))\label{eq:vsfat}\\
    \fatr(\MGS{\CCE}[\cA,\cA_p],2\epsilon)&\leq (1/\epsilon)\fatr(G,\epsilon)\label{eq:vfatr}
\end{align}
where \eqref{eq:vfat} follows from Lemma \ref{lem:vfat}, \eqref{eq:vsfat} follows from Lemma \ref{lem:vsfat}, and \eqref{eq:vfatr} follows from Lemma \ref{lem:vfatr}.  Then,
\begin{align}
    \fat(\MGS{\CCE} ,16\epsilon)&\leq O(k\fat(G,\epsilon)\log\fat(G,\epsilon) \log^2(1/\epsilon)) \label{eq:hfat}\\
    \sfat(\MGS{\CCE},2\epsilon)&\leq O((k/\epsilon)\sfat(G,\epsilon))\label{eq:hsfat}\\
    \fatr(\MGS{\CCE},2\epsilon)&\leq (k/\epsilon)\fatr(G,\epsilon)\label{eq:hfatr}
\end{align}
where \eqref{eq:hfat} follows from Lemma \ref{lem:hfat}, \eqref{eq:hsfat} follows from Lemma \ref{lem:hsfat}, and \eqref{eq:hfatr} follows from Lemma \ref{lem:hfatr}. 
\end{proof}

For item \ref{it:app-best-resp}, we need to be able to compute two things.  We have a zero-sum matrix game on $\MGS{\CCE}: \cA \times \lr{\bigcup_p \cA_p} \to [-1,1]$ between a minimizing player selecting distributions over action profiles $\mu \in \Delta \lrset{\cA}$ and a maximizing player selecting distributions over deviations $\xi \in \bigcup_p \cA_p$.  The corresponding ``$\epsilon$-best-response'' oracles would return, for all $\xi \in \Delta\lr{\bigcup_p \cA_p}$, $\hat a=\texttt{BestResponse}(\xi,\epsilon) \in \cA$ satisfying
\begin{align}
    \mathbbm{1}[\hat a]^T\MGS{\CCE}\xi \leq \inf_{a \in \cA}\mathbbm{1}[a]^T\MGS{\CCE}\xi+\epsilon \label{eq:cce-best-resp-1}
\end{align}
and, for all $\mu \in \Delta(\cA)$, $(\hat p,d_{\hat p}) = \texttt{BestResponse}(\mu,\epsilon) \in \bigcup_p \cA_p$ satisfying
\begin{align}
    \mu^T\MGS{\CCE}\mathbbm{1}[(\hat p, d_{\hat p})] \geq \sup_{(p,a_p) \in \bigcup_p \cA_p}\mu^T\MGS{\CCE}\mathbbm{1}[(p,a_p)]-\epsilon \label{eq:cce-best-resp-2}
\end{align}

As stated in Section \ref{sec:CCE-main}, \eqref{eq:cce-best-resp-1} is not necessary.  It is used at time step $t$ in Algorithm~\ref{alg:nash} when the minimizing player is increasing her set of actions $A_{t-1} \to A_t$ in order to improve her value in the game versus the maximizing player's current set of actions $B_{t-1}$.  For the specific case of the matrix game $\MGS{CCE}$, we have the following construction of $A_t$ in terms of $B_{t-1}$.

\begin{algorithm}[H]
\caption{\texttt{Finite-CCE}}
\label{alg:prod}
\textbf{Input:} A finite game $((B^1,\dots,B^k),u)$, an $\epsilon>0$\\ 
\textbf{Subroutines:}
\begin{itemize}
    \item Multiplicative-weight update: a no-regret learning algorithm that operates on a finite set of actions (Algorithm~\ref{alg:MW})
\end{itemize}
\begin{enumerate}
    \item Simulate a repeated game between the $k$ players, where each player plays according to multiplicative weight update (Algorithm~\ref{alg:MW}), for $T = \Theta\lr{\log\lr{\max_p |B_p|}/\epsilon^2}$ iterations
    \item \textbf{Return} the uniform distribution over $\{a^1,\dots,a^T\}$ where $a^t$ is the strategy profile played by the players at iteration $t$.
\end{enumerate}
\end{algorithm}

\noindent From Lemma~\ref{lem:equi-via-online}, we have that the output of Algorithm~\ref{alg:prod} will constitute an $\epsilon$-$\CCE$ for the subgame $G'=(\prod_p B_{t-1,p},u)$.  
Therefore, we have
\begin{equation}\label{eq:prod}
    \Val_{\MGS{\CCE}}\lr{\texttt{Support}(\texttt{Finite-CCE}(B_{t-1})),B_{t-1}} \leq \epsilon
\end{equation}
By the definition of $\epsilon$-$\CCE$, there is no profitable deviation for any player to any strategy within the subgame $G'$. Therefore, the outputted distribution $\mu$ must satisfy $\sup_{\xi \in \Delta\lr{B_{t-1}}} \mu^T \MGS{\CCE} \xi \leq \epsilon$, and Player 1 can force the game to have value $\leq \epsilon$ using a strategy supported on $\texttt{Support}(\texttt{Finite-CCE}(B_{t-1}))$.  This suffices for the purposes of our algorithm.  To achieve \eqref{eq:cce-best-resp-2}, we introduce the following subroutine

\begin{algorithm}[H]
\caption{\texttt{$\epsilon$-best deviation for an action profile distribution}}
\label{alg:best-dev}
\textbf{Input:} A bounded-support distribution over action profiles $\mu \in \Delta(\cA)$, a value $\epsilon>0$\\
\textbf{Subroutines:}
\begin{itemize}
    \item $\texttt{BestResponse}_p$ oracle: for a player $p$, receives a bounded-support distribution over adversary actions $\mu_{-p} \in \Delta(\cA_{-p})$ and an $\epsilon>0$ and outputs an $\epsilon$-best response from $\cA_p$ (see Definition~\ref{def:bf}).
    \item Utility functions $u_p(a_p,a_{-p})$ for all players $p$ that receive an action profile and outputs player $p$'s utility
\end{itemize}
\begin{enumerate}
    \item \textbf{For} $p=1,2,\dots,k$
    \begin{enumerate}
        \item Marginalize: $\mu_{-p}[a_{-p}] \gets \sum_{a_p \in \texttt{Support}(\mu)_p} \mu[(a_p,a_{-p})]$ \textbf{for} $a_{-p} \in \texttt{Support}(\mu)_{-p}$
        \item Best deviation for $p$: $d_p \gets \texttt{BestResponse}_p(\mu_{-p},\epsilon)$
        \item Value of not deviating for $p$: $v_p \gets \sum_{a \in \texttt{Support}(\mu)} u_p(a_p,a_{-p})\mu[a]$
        \item Value of best deviation for $p$: $v'_p \gets \sum_{a_{-p} \in \texttt{Support}(\mu)_{-p}} u_p(d_p,a_{-p})\mu_{-p}[a_{-p}]$
    \end{enumerate}
    \item Player with greatest value increase if she deviates: $\hat p \gets \arg\max_{p \in [k]} v'_p - v_p$ 
    \item \textbf{Return} $(\hat p, d_{\hat p})$
\end{enumerate}
\end{algorithm}

\noindent The following lemma demonstrates that this subroutine gives the desired $\epsilon$-best-response of \eqref{eq:cce-best-resp-2}.

\begin{lemma}[Best deviation for an action profile distribution]\label{lem:best-dev}
    Consider a game\\
    $G=(\cA=\prod_{p=1}^k \cA_p,u=(u_1,\cdots,u_k))$, bounded-support distribution over action profiles $\mu \in \Delta(\cA)$, and a value $\epsilon>0$.  Then, Algorithm~\ref{alg:best-dev} executed with parameters $\mu,\epsilon$ outputs a $(\hat p,d_{\hat p})$ satisfying
    \begin{equation*}
        \mu^T\MGS{\CCE}\mathbbm{1}[(\hat p,d_{\hat p})] \geq \sup_{(p,a_p) \in \bigcup_p \cA_p}\mu^T\MGS{\CCE}\mathbbm{1}[(p,a_p)]-\epsilon
    \end{equation*}
\end{lemma}
\begin{proof}
Recalling Definition \ref{def:bf}, our $\epsilon$-\emph{best-response} oracle for player $p$ receives a bounded-support distribution $\mu_{-p} \in \Delta(\mathcal{A}_{-p})$ and outputs an action $d_p \in \cA_p$ satisfying
$$\sum_{a_{-p} \in \cA_{-p}} u_p(d_p,a_{-p}) \mu_{-p}[a_{-p}] \ge \sup_{ a_p \in \cA_p} \lr{\sum_{a_{-p} \in \cA_{-p}} u_p( a_p,a_{-p}) \mu_{-p}[a_{-p}]} - \epsilon$$
Therefore, for all $p$, the values $v_p,v'_p$ in Algorithm~\ref{alg:best-dev} satisfy 
\begin{align*}
v_p'-v_p&=\sum_{a_{-p} \in \cA_{-p}}\sum_{a_p \in \cA_p} \Lr{u_p(d_p,a_{-p})-u_p(a_p,a_{-p})} \mu[(a_p,a_{-p})]\\
&\ge \sup_{ a_p \in \cA_p} \lr{\sum_{a_{-p} \in \cA_{-p}} \sum_{a_p \in \cA_p} \Lr{u_p( a_p,a_{-p})-u_p(a_p,a_{-p})} \mu[(a_p,a_{-p})]} - \epsilon
\end{align*}
or equivalently
\begin{align*}
v_p'-v_p= \mu^T \MGS{\CCE} \mathbbm{1}[(p,d_p)] \ge \sup_{ a_p \in \cA_p} \mu^T \MGS{\CCE} \mathbbm{1}[(p, a_p)] - \epsilon
\end{align*}
Therefore, defining $p^*$ such that $\widehat{(p,a_{p})} \in \cA_{p^*}$ where
$$\widehat{(p,a_{p})} = \arg\sup_{(p,a_{p}) \in \bigcup_p \cA_p} \mu^T \MGS{\CCE} \mathbbm{1}[(p,a_{p})]$$
we have
\begin{align*}
\mu^T \MGS{\CCE} \mathbbm{1}[(\hat p,d_{\hat p})] & \ge \mu^T \MGS{\CCE} \mathbbm{1}[(p^*,d_{p^*})]\\
&\ge \sup_{a_{p^*} \in \cA_{p^*}} \mu^T \MGS{\CCE} \mathbbm{1}[(p,a_{p})] - \epsilon\\
&= \sup_{(p,a_{p}) \in \bigcup_p \cA_p} \mu^T \MGS{\CCE} \mathbbm{1}[(p,a_{p})] - \epsilon
\end{align*} as desired.
\end{proof}

We are now ready to state our main algorithm and prove our main theorem.

\begin{algorithm}[H]
\caption{\texttt{$O(\epsilon)$-approximate $\CCE$ for a general-sum game}}
\label{alg:CCE}
\textbf{Input:} A general-sum game $G=(\cA=\prod_{p=1}^k \cA_p,u=(u_1,\cdots,u_k))$, a parameter $\epsilon>0$ \\
\textbf{Subroutines:}
\begin{itemize}
    \item \texttt{BestDeviation}: Receives a bounded-support distribution over action profiles $\mu \in \Delta(\cA)$, and a value $\epsilon>0$.  Returns an $\epsilon$-best-response from $\bigcup_p \cA_p$ for the matrix game on $\MGS{\CCE}$
    \item \texttt{Nash}: Receives finite set of action profiles $A \subset \cA$, and an $\epsilon>0$. Returns an $\epsilon$-Nash for the two-player zero-sum subgame $(A,\bigcup_p \cA_p,\MGS{\CCE})$ using \texttt{BestDeviation} as its \texttt{BestResponse} subroutine(Algorithm~\ref{alg:nash-half-inf})
    \item \texttt{FiniteCCE}: Receives a finite game $((B_1,\cdots,B_k),u)$ and returns an $\epsilon$-$\CCE$ 
    \item $\Val$: Receives finite sets of action profiles $A \subset \cA$ and of deviations $B \subset \bigcup_p \cA_p$. Returns the value of this finite subgame on $\MGS{\CCE}$.
\end{itemize}
\begin{enumerate}
    \item $A_0 \gets \emptyset$, $B_0 \gets \{(1,a_1),\cdots,(k,a_k)\}$, where $(a_1,\cdots,a_k)\in \cA$ are arbitrary actions
    \item \textbf{For} $t=1,2,\dots$
    \begin{enumerate}
        \item $A_t \gets A_{t-1} \cup \texttt{Support}(\texttt{FiniteCCE}((B_{t-1,1},\dots,B_{t-1,k}),u))$ where $B_{t-1,p} = B_{t-1} \cap \cA_p$
        \item  $(\xi^{t,1},\xi^{t,2}) \gets \texttt{Nash}\l( A_t,\bigcup_p \cA_p, \epsilon\r)$
        \item $B_t \gets B_{t-1}\cup \texttt{Support}(\xi^{t,2})$
        \item \textbf{if}
        $\Val(A_t,B_{t}) \le 3\epsilon$
        \begin{enumerate}
            \item \textbf{Return} $\xi^{t,1}$
        \end{enumerate}
    \end{enumerate}
\end{enumerate}
\end{algorithm}

As stated previously, this algorithm is analogous to Algorithm~\ref{alg:nash} run on the game matrix $\MGS{\CCE}$, with a modification to how the minimizing player adds actions to her support (using \texttt{FiniteCCE}).  Analogously, we have the following theorem.

In a similar fashion with section \ref{sec:alg-nash}, we will use three main results to prove Theorem \ref{thm:CCE}, bounding the total number of oracle calls made by Algorithm~\ref{alg:CCE}.  First, we show that the output of Algorithm~\ref{alg:CCE}, given that it stops, constitutes an $O(\epsilon)$-$\CCE$(Lemma~\ref{lem:out-is-CCE}).  Then, we show that Algorithm~\ref{alg:CCE} necessarily stops after $\fatr(\conv(\MGS{\CCE}),2\epsilon)$ iterations (Lemma~\ref{lem:CCE-iterations-bound}). Last, we show that Algorithm~\ref{alg:CCE} running for $T$ iterations makes only $O(k T/\epsilon^2 \cdot \log(T/\epsilon^2))$ oracle calls (Lemma~\ref{lem:CCE-bnd-oracle-calls}).

\begin{lemma}
    \label{lem:out-is-CCE}
    Assume that  Algorithm~\ref{alg:CCE} stops. Then, the returned strategies constitute a $5\epsilon$-$\CCE$ for the original game $G=(\cA, u)$.
\end{lemma}

\begin{proof}
    The returned $\xi^{t,1}$ is the strategy for player $1$ in an $\epsilon$-Nash for the game $(A_t,\bigcup_p \cA_p,\MGS{\CCE})$.  So, for every deviation $(p,d_p) \in \bigcup_p \cA_p$:
\[
\MGS{\CCE}[\xi^{t,1},(p,d_p)] \le \Val(A_t,\bigcup_p \cA_p) + \epsilon
\]
    Since $\xi^{t,2}$ is the strategy for player $2$ in an $\epsilon$-Nash for the game $(A_t,\bigcup_p \cA_p,\MGS{\CCE})$, we have:
\[
\Val(A_t,\bigcup_p \cA_p) \le \Val(A_t,B_t) + \epsilon.
\]
    Lastly, by the stopping condition of Algorithm~\ref{alg:CCE}
\[
\Val(A_t,B_t) \le 3\epsilon.
\]
    Combining these 3 equations gives the desired $\MGS{\CCE}[\xi^{t,1},(p,d_p)] \le 5\epsilon$ for all $(p,d_p) \in \bigcup_p \cA_p$.
\end{proof}

\begin{lemma}
    \label{lem:CCE-iterations-bound}
    Algorithm \ref{alg:CCE} terminates after $\fatr(\conv(\MGS{\CCE}),2\epsilon)$ iterations.
\end{lemma}

\begin{proof}
    Suppose the algorithm runs for $\geq T$ iterations.  Then, we will have support sets $A_1 \subseteq A_2 \subseteq \cdots \subseteq A_T$ and $B_1 \subseteq B_2 \subseteq \cdots \subseteq B_T$ satisfying the following.  For all $1 \leq i,j \leq T$,

    \begin{equation*}
    \begin{aligned}
    \Val_{\MGS{\CCE}}(A_{i}, B_{j}) &\le \epsilon  \qquad \text{ if } i>j \\
    \Val_{\MGS{\CCE}}(A_{i}, B_{j}) &\ge 3\epsilon  \qquad \text{ if } i\leq j
    \end{aligned}
    \end{equation*}

    The first holds due to \eqref{eq:prod} and the fact that $\texttt{Support}(\texttt{Finite-CCE}(B_{t-1}))\subseteq A_t$. The second holds because the continuation of the algorithm implies the stopping condition is not met.  Defining $\alpha_i$ to be the minmax strategy of Player 1 in the subgame $(A_i,B_{i-1},\MGS{\CCE})$ and $\beta_i$ to be the minmax strategy of Player 2 in the subgame $(A_i,B_{i},\MGS{\CCE})$ for all $i$, we have 

    \begin{equation*}
    \begin{aligned}
    \MGS{\CCE}[\alpha_{i}, \beta_{j}] &\le \epsilon  \qquad \text{ if } i>j \\
    \MGS{\CCE}[\alpha_{i}, \beta_{j}] &\ge 3\epsilon  \qquad \text{ if } i\leq j
    \end{aligned}
    \end{equation*}

    \noindent which constitutes a $(2\epsilon)$-fat-thresholding matrix in $\MGS{\CCE}$.  Therefore, $\fatr(\conv(\MGS{\CCE}),2\epsilon) \geq T$ as desired.
\end{proof}

In order to bound the number of oracle calls, we add the following lemma:
\begin{lemma}\label{lem:CCE-bnd-oracle-calls}
    Assume that Algorithm~\ref{alg:CCE} runs for $T$ iterations. Then, the number of oracle calls is bounded by $\cO(k T/\epsilon^2 \cdot \log(T/\epsilon^2))$.
\end{lemma}
\begin{proof}
First, we would like to bound the sizes of $A_t$ and $B_t$ by $\cO(t/\epsilon^2\cdot \log(t/\epsilon^2))$. In order to show that, notice that $B_{t}$ is obtained from $B_{t-1}$ by adding the support of an $\epsilon$-approximate Nash for the half-infinite game $(A_t, \bigcup_p \cA_p, \MGS{\CCE})$. We would like to bound the size of the support of the strategy of Player 2 in such an approximate Nash. This approximate Nash is computed in Algorithm~\ref{alg:nash-half-inf}, and the size of the support equals the number of iterations of this algorithm, which is bounded by $\cO(\log |A_{t}|/\epsilon^2)$, by Lemma~\ref{lem:oracle-reduction}. Further, we would like to argue that $|A_t| \le |A_{t-1}| + C (\log |B_{t-1}|+C)/\epsilon^2$ for a universal constants $C>0$. Indeed, this is true since $A_t$ is obtained from $A_{t-1}$ by adding the support of a CCE computed by Algorithm~\ref{alg:prod}, given an action-set taken from $B_{t-1}$, and the support size is bounded by $\cO(\log(|B_t|)/\epsilon^2)$. By an inductive argument, it is easy to show that these two recursive equations for $A_t$ and $B_t$ imply that $|A_t|,|B_t|\le \cO\lr{t/\epsilon^2 \cdot \log(t/\epsilon^2)}$. 

Lastly, it remains to bound the number of oracle calls. Notice that any addition of an action to the support of Player $2$ involves computing a best deviation (Algorithm~\ref{alg:best-dev}), which is being used as a subroutine in this instance of the half-infinite equilibrium computation (Algorithm~\ref{alg:nash-half-inf}).  Algorithm~\ref{alg:best-dev} makes exactly $k$ calls to the best response oracle.  Thus, the total number of calls is bounded by $\cO\lr{kt/\epsilon^2 \cdot \log(t/\epsilon^2)}$.
\end{proof}

We are now ready to prove our main theorem.  Recall

\noindent \textbf{Theorem \ref{thm:CCE}}
\textit{
Let $G=(\actions=\mathcal{A}_1 \times \cdots \times \mathcal{A}_k,u=(u_1,\cdots,u_k))$ be a multi-player game. Assume that utilities are bounded $u_p: \cA \to [0,1]$, and let $\epsilon>0$. Then, Algorithm~\ref{alg:CCE} executed with parameters $G,\epsilon$ will compute an $O(\epsilon)-\CCE$ for the game using using the following number of $\epsilon$-best response oracle calls:
\begin{align*}
\cO\lr{
~e^{C(k/\epsilon^3)\sfat(G,\epsilon/C)} }
\end{align*}
}

\begin{proof}[Proof of Theorem~\ref{thm:CCE}]
We notice that Lemma~\ref{lem:out-is-CCE} implies that the output of Algorithm \ref{alg:CCE} is an $O(\epsilon)$-$\CCE$. Furthermore, Lemma~\ref{lem:CCE-iterations-bound} bounds the number of iterations by $T \le \fatr(\conv(\MGS{\CCE}),2\epsilon)$ and Lemma~\ref{lem:CCE-bnd-oracle-calls} implies that the number of oracle calls is bounded by $\widetilde\cO\lr{(k/\epsilon^2)\fatr(\conv(\MGS{\CCE}),2\epsilon)}$. The proof of Theorem~\ref{thm:CCE} follows by substituting $\fatr(\conv(\MGS{\CCE}),2\epsilon)$ according to Theorem~\ref{thm:conv-bnd}:
\[
\fatr(\conv(\MGS{\CCE}),2\epsilon) \le e^{C\sfat(\MGS{\CCE},2\epsilon/C)/\epsilon^2}\enspace.
\]
The total bound on the number of oracle calls is then
\begin{align*}
&\widetilde\cO\lr{(k/\epsilon^2)\fatr(\conv(\MGS{\CCE}),2\epsilon)}\\
&\le (k/\epsilon^2) e^{C\sfat(\MGS{\CCE},2\epsilon/C)/\epsilon^2} 
\end{align*}
Notice that the fact of $k/\epsilon^2$ can be omitted by changing the constant in the exponent.  Lastly, plugging in our bounds on the dimensions of $\MGS{\CCE}$ in terms of those of $G$ from Lemma~\ref{lem:CCE-mat-dim-bound}, we get that the number of oracle calls made by Algorithm~\ref{alg:CCE} is
\begin{align*}
\le e^{C(k/\epsilon^3)\sfat(G,\epsilon/C)} 
\end{align*}
as desired.
\end{proof}

\section{Concatenation Lemmas}\label{app:concat}

\begin{lemma}[Vertical $\epsilon$-fat-shattering concatenation tool]\label{lem:vfat}
    For each $p \in [k]$, say we have a \rfc~ $\cF_p$ defined on a domain set $X$ with $\fat(\cF_p,\epsilon)\leq d$.  Then, the ``vertically-concatenated'' \rfc~ $\bigcup_p \cF_p$ has $\fat(\bigcup_p \cF_p,8\epsilon)= O\lr{\log (k) + d \log (d) \log^2(1/\epsilon)}$.
\end{lemma}

\begin{proof}
For a \rfc~ $\cF$, define the $\epsilon$-covering growth function
    \begin{equation}
        \cN_{\cF, \epsilon}(n) = \max_{\substack{S \subseteq X\\|S|=n}} \: \min \lrset{|V| : \text{V is an $\epsilon$-cover, under the $\infty$-norm, of $\cF$ on $S$}}
    \end{equation}
    That is, we want to find the smallest $V \subseteq \hR^S$ such that, for all $f \in \cF$, there exists $v \in V$ such that, for all $x \in S$, $|f(x)-v(x)| \leq \epsilon$.  For each $p \in [k]$, if $V_p$ is an $\epsilon$-cover of $\cF_p$, then $\bigcup_p V_p$ $\epsilon$-covers $\bigcup \cF_p$.  Therefore,
    \begin{equation*}
        \cN_{\bigcup_p \cF_p,\epsilon}(n) \leq \sum_p \cN_{\cF_p,\epsilon}(n)
    \end{equation*}

    \noindent From Theorem 1.5 of \cite{lecture_notes}, if $\fat(\cF,\epsilon)\leq d$, then
    \begin{equation}\label{eq:fat-sauer}
        \cN_{\cF,4\epsilon}(n) \leq O\lr{\frac{n}{\epsilon^2}}^{\lceil d \log\lr{\frac{en}{d\epsilon}} \rceil}
    \end{equation}
    \noindent Given this, we can argue the following about the vertically-concatenated \rfc.  Since a union of covers of the $\cF_p$ for each $p$ will cover $\bigcup_p \cF_p$,
    \begin{align*}
        \cN_{\bigcup_p \cF_p,4\epsilon}(n) &\leq \sum_p \cN_{\cF_p,4\epsilon}(n)\leq k \: O\lr{\frac{n}{\epsilon^2}}^{\lceil d \log\lr{\frac{en}{d\epsilon}} \rceil}
    \end{align*}
    and for $n = \Omega\lr{\log (k) + d \log (d) \log^2(1/\epsilon)}$
    \begin{align*}
        \log \lr{k \: O\lr{\frac{n}{\epsilon^2}}^{\lceil d \log\lr{\frac{en}{d\epsilon}} \rceil}} &= \log (k) + d \lr{1+\log(n/d)+\log(1/\epsilon)}\lr{\log(n) + 2 \log(1/\epsilon)} = o(n)
    \end{align*}
    Therefore, for $n = C\lr{\log (k) + d \log (d) \log^2(1/\epsilon)}$ with sufficiently large constant $C$, $$\cN_{\bigcup_p \cF_p,4\epsilon}(n) < 2^n$$
    Therefore, for all $S = \lrset{x_1,\cdots,x_n } \subseteq X$ and all $\theta \in \hR^n$, there exists $b \in \lrset{0,1}^n$ such that there is no $f \in \bigcup_p \cF_p$ with
\begin{equation}
\begin{aligned}
        f(x_j) &> \theta_j+8\epsilon \qquad &\text{for all }j \in [n] \text{ with }b_j=1\\
        f(x_j) &\leq \theta_j \qquad &\text{for all }j \in [n] \text{ with }b_j=0
    \end{aligned}
\end{equation}
If there existed such an $f_b \in \bigcup_p \cF_p$ for every $b$, each would have to be covered by a distinct $v$ in the minimal $(4\epsilon)$-cover $V$.  This would force the size of the cover to be at least $2^n$, a contradiction.  Therefore, $\fat(\bigcup_p \cF_p,8\epsilon)= O\lr{\log (k) + d \log (d) \log^2(1/\epsilon)}$, as desired.\\
\end{proof}

\begin{lemma}[Horizontal $\epsilon$-fat-shattering concatenation tool]\label{lem:hfat}
    Say we have \rfc es $\cF_p$ for all $p \in [k]$ defined on mutually-disjoint domain sets $X_p$ with $\fat(\cF_p,\epsilon)\leq d$.  Let's say these classes all have the same magnitude, and their elements are enumerated by a set $\cA$.  That is, for all $p$, there is an $f_{p,a} \in \cF_p$ for each $a \in \cA$.  We define the ``horizontal concatenation'' function class $\cF = \lrset{f_a|a \in \cA}$ where each $f_a: \bigcup_p X_p \to \hR$ is defined
    \begin{equation}
        f_a(x) = f_{p,a}(x) \qquad \text{for all }x \in X_p \text{ for all }p
    \end{equation}
    Then, the horizontal concatenation function class $\cF$ has $\fat(\cF,\epsilon) \leq kd$.
\end{lemma}

\begin{proof}
Assume for the sake of contradiction that there exists $S = \lrset{x_1,\cdots,x_{kd+1}} \subseteq \bigcup_p X_p$ and witnesses $\theta_1,\cdots,\theta_{kd+1}$ such that for every $b \in \lrset{0,1}^{kd+1}$, there exists $a \in \cA$ with
\begin{equation}
\begin{aligned}
        f_a(x_j) &\geq \theta_j+\epsilon \qquad &\text{for all }j \in [kd+1] \text{ with }b_j=1\\
        f_a(x_j) &\leq \theta_j \qquad &\text{for all }j \in [kd+1] \text{ with }b_j=0
    \end{aligned}
\end{equation}
By the pigeon hole principle, there must exist a $p \in [k]$ such that $|S \cap X_p| \geq d+1$.  That would imply, for every $b \in \lrset{0,1}^{|S \cap X_p|}$, there exists $a \in \cA$ with
\begin{equation}
\begin{aligned}
        f_{p,a}(x_j) &\geq \theta_j+\epsilon \qquad &\text{for all }j \text{ with }x_j\in S \cap X_p \text{ and }b_j=1\\
        f_{p,a}(x_j) &\leq \theta_j \qquad &\text{for all }j \text{ with }x_j\in S \cap X_p \text{ and }b_j=0
    \end{aligned}
\end{equation}
contradicting our assumption that $\fat(\cF_p,\epsilon) \leq d$, as desired.
\end{proof}

\begin{lemma}[Vertical $\epsilon$-sequential-fat-shattering concatenation tool]\label{lem:vsfat}
    For each $p \in [k]$, say we have a \rfc~ $\cF_p$ defined on a domain set $X$ with $\sfat(\cF_p,\epsilon)\leq d$.  Then, the ``vertically-concatenated'' \rfc~ $\bigcup_p \cF_p$ has $\sfat(\bigcup_p\cF_p,\epsilon)\leq k(d+1)-1$.
\end{lemma}
\begin{proof}
    Assume for the sake of contradiction there exists a complete binary tree $T=(V,E)$ of depth $k(d+1)$, whose internal nodes $v\in V$ are labeled by elements $x(v) \in X$ and have witnesses $\theta(v)$, and whose leaves $\ell \in V$ are labeled by $f_\ell\in \bigcup_p \cF_p$, such that the following holds: for any root-to-leaf path $v_1,\dots,v_{k(d+1)},v_{k(d+1)+1}=\ell$ in the tree and for any $i \in [k(d+1)]$:
    \begin{align*}
        f_\ell(x(v_i)) &\geq \theta(v_i)+\epsilon & \text{if $v_{i+1}$ is a \emph{left} child of $v_i$}\\
        f_\ell(x(v_i)) &\leq \theta(v_i) & \text{if $v_{i+1}$ is a \emph{right} child of $v_i$}
    \end{align*}
    We color the leaf nodes of the tree with $k$ colors where $c(\ell) = p$ iff $f_\ell \in \cF_p$.  We will demonstrate that there exists a complete binary ``subtree'' of depth $d+1$ that is leaf-monochromatic.  Here, we define a subtree to be a tree $T'=(V',E')$ on a subset of nodes $V' \subseteq V$ where every internal node $v \in V'$ satisfies:
    \begin{center}
        $v_{\text{left}}$ is the left child of $v$ in $T'$ implies $v_{\text{left}}$ is a left descendant of $v$ in $T$\\
        $v_{\text{right}}$ is the right child of $v$ in $T'$ implies $v_{\text{right}}$ is a right descendant of $v$ in $T$\\
    \end{center}
    We also ensure that leaves in $T'$ are leaves in $T$.  The existence of this depth-$(d+1)$ complete leaf-monochromatic subtree would give the desired contradiction, implying that for some $p$, $\sfat(\cF_p) \geq d+1$. To prove the existence of this subtree, we will use the following lemma.
    \begin{lemma}\label{lem:leaf-tree-ramsey}
        Define $R_{\text{leaf}}(C_1,\cdots,C_k)$ to be the minimum integer such that any complete binary tree of depth $R_{\text{leaf}}(C_1,\cdots,C_k)$ with $k$-colored leaves necessarily has, for some $p \in [k]$, a complete binary subtree of depth $C_p$ with all leaves colored $p$.  Then,
        \begin{equation*}
            R_{\text{leaf}}(C_1,\cdots,C_k) \leq \sum_{p} C_p
        \end{equation*}
    \end{lemma}
    From the lemma, $R_{\text{leaf}}(d+1,\cdots,d+1) \leq k(d+1)$ and therefore the desired leaf-monochromatic depth-$(d+1)$ subtree exists.
\end{proof}

\begin{proof}[Proof of Lemma \ref{lem:leaf-tree-ramsey}]
    We prove by induction on $\sum_{p} C_p$.  For a base case, note that $R_{\text{leaf}}(0,\cdots,0)=0$.  For a depth-$0$ binary tree (consisting of 1 leaf node), whatever color we select for the leaf, the leaf itself will constitute the desired depth-$0$ leaf-monochromatic subtree.\\

    Assume the claim holds for all $R_{\text{leaf}}(C'_1,\cdots,C'_k)$ with $\sum_{p} C'_p\leq d$. Let's consider a $k$-leaf-colored depth-$(d+1)$ tree and $\sum_{p} C_p = d+1$.  We note that the two child trees of the root have depth $d \geq R_{\text{leaf}}(C_1,\cdots,C_p-1,\cdots,C_k)$ for an arbitrarily selected color $p$.  If either child tree contains a $q$-leaf-monochromatic subtree of depth $C_q$ for some $q \ne p$, we are done.  So, from the inductive hypothesis, assume both child trees contain $p$-leaf-monochromatic subtrees of depth $C_p-1$.  These two trees together with the root give the desired $p$-leaf-monochromatic subtree of depth $C_p$.
\end{proof}

\begin{lemma}[Horizontal $\epsilon$-sequential-fat-shattering concatenation tool]\label{lem:hsfat}
    Say we have \rfc es $\cF_p$ for all $p \in [k]$ defined on mutually-disjoint domain sets $X_p$ with $\sfat(\cF_p,\epsilon)\leq d$.  Let's say these classes all have the same magnitude, and their elements are enumerated by a set $\cA$.  That is, for all $p$, there is an $f_{p,a} \in \cF_p$ for each $a \in \cA$.  We define the ``horizontal concatenation'' function class $\cF = \lrset{f_a|a \in \cA}$ where each $f_a: \bigcup_p X_p \to \hR$ is defined
    \begin{equation*}
        f_a(x) = f_{p,a}(x) \qquad \text{for all }x \in X_p \text{ for all }p
    \end{equation*}
    Then, the horizontal concatenation function class $\cF$ has $\sfat(\cF,\epsilon) \leq kd$.
\end{lemma}

\begin{proof}
    Assume for the sake of contradiction there exists a complete binary tree $T=(V,E)$ of depth $kd+1$, whose internal nodes $v\in V$ are labeled by elements $x(v) \in \bigcup_p X_p$ and have witnesses $\theta(v)$, whose leaves $\ell \in V$ are labeled by $f_\ell\in \cF$, such that the following holds: for any root-to-leaf path $v_1,\dots,v_{kd+1},v_{kd+2}=\ell$ in the tree and for any $i \in [kd+1]$:
    \begin{align*}
        f_\ell(x(v_i)) &\geq \theta(v_i)+\epsilon & \text{if $v_{i+1}$ is a \emph{left} child of $v_i$}\\
        f_\ell(x(v_i)) &\leq \theta(v_i) & \text{if $v_{i+1}$ is a \emph{right} child of $v_i$}
    \end{align*}
    We color the internal nodes of the tree with $k$ colors where $c(v) = p$ iff $x(v) \in X_p$.  We will demonstrate that there exists a complete binary subtree of depth $d+1$ that is internally-monochromatic.  
    The existence of this complete internally-monochromatic subtree would give the desired contradiction, implying that for some $p$, $\sfat(\cF_p) \geq d+1$. To prove the existence of this subtree, we will use the following lemma.
    \begin{lemma}\label{lem:tree-ramsey}
        Define $R_{\text{int}}(C_1,\cdots,C_k)$ to be the minimum integer such that any complete binary tree of depth $R_{\text{int}}(C_1,\cdots,C_k)$ with $k$-colored internal nodes necessarily has, for some $p \in [k]$, a complete binary subtree of depth $C_p$ with all internal nodes colored $p$.  Then,
        \begin{equation*}
            R_{\text{int}}(C_1,\cdots,C_k) \leq \sum_{p} C_p - k + 1
        \end{equation*}
    \end{lemma}
    From the lemma, $R_{\text{int}}(d+1,\cdots,d+1) \leq kd+1$ and therefore the desired internally-monochromatic depth-$(d+1)$ subtree exists.
\end{proof}

\begin{proof}[Proof of Lemma \ref{lem:tree-ramsey}]
    We prove by induction on $\sum_{p} C_p$.  For a base case, note that $R_{\text{int}}(1,\cdots,1)=1$.  For a depth-$1$ binary tree (consisting of 1 internal node and 2 leaves), whatever color we select for the internal node, the tree itself will constitute the desired depth-$1$ internally-monochromatic subtree (leaves are not colored).\\

    Assuming the claim holds for all $R_{\text{int}}(C'_1,\cdots,C'_k)$ with $\sum_{p} C'_p\leq d$, let's consider a $k$-colored depth-$(d+1)$ tree.  Assume without loss of generality that the root has color $p$.  We note that its two child trees have depth $d \geq R_{\text{int}}(C_1,\cdots,C_p-1,\cdots,C_k)$.  If either child tree contains a $q$-internally-monochromatic subtree of depth $C_q$ for some $q \ne p$, we are done.  So, from the inductive hypothesis, assume both child trees contain $p$-internally-monochromatic subtrees of depth $C_p-1$.  These two trees together with the root of color $p$ give the desired $p$-interally-monochromatic subtree of depth $C_p$.
\end{proof}

\begin{lemma}[Vertical $\epsilon$-fat-threshold concatenation tool]\label{lem:vfatr}
    For each $p \in [k]$, say we have a \rfc~ $\cF_p$ defined on a domain set $X$ with $\fatr(\cF_p,\epsilon)\leq d$.  Then, the ``vertically-concatenated'' \rfc~ $\bigcup_p \cF_p$ has $\fatr(\bigcup_p\cF_p,\epsilon) \leq kd$.
\end{lemma}

\begin{proof}
Assume for the sake of contradiction that there exists $F = \lrset{f_1,\cdots,f_{kd+1}} \subseteq \bigcup_p \cF_p$, $S = \lrset{x_1,\cdots,x_{kd+1}} \subseteq X$, and witness $\theta$ such that for every $i,j \in [kd+1]$
\begin{equation}
\begin{aligned}
        f_{i}(x_j) &\geq \theta+\epsilon \qquad &\text{for all }i\leq j\\
        f_{i}(x_j) &\leq \theta \qquad &\text{for all }i> j
    \end{aligned}
\end{equation}

By the pigeon hole principle, there must exist a $p \in [k]$ such that $|F \cap \cF_p| \geq d+1$.  Therefore, for every $i,j$ with $f_i,f_j \in F \cap \cF_p$

\begin{equation}
\begin{aligned}
    f_{i}(x_j) &\geq \theta+\epsilon \qquad &\text{for all }i\leq j\\
    f_{i}(x_j) &\leq \theta \qquad &\text{for all }i> j
\end{aligned}
\end{equation}

\noindent contradicting our assumption that $\fatr(\cF_p,\epsilon) \leq d$, as desired.
\end{proof}

\begin{lemma}[Horizontal $\epsilon$-fat-threshold concatenation tool]\label{lem:hfatr}
    Say we have \rfc es $\cF_p$ for all $p \in [k]$ defined on mutually-disjoint domain sets $X_p$ with \eftd~ $\leq d$.  Let's say these classes all have the same magnitude, and their elements are enumerated by a set $\cA$.  That is, for all $p$, there is an $f_{p,a} \in \cF_p$ for each $a \in \cA$.  We define the ``horizontal concatenation'' function class $\cF = \lrset{f_a|a \in \cA}$ where each $f_a: \bigcup_p X_p \to \hR$ is defined
    \begin{equation}
        f_a(x) = f_{p,a}(x) \qquad \text{for all }x \in X_p \text{ for all }p
    \end{equation}
    Then, the horizontal concatenation function class $\cF$ has \eftd~ $\leq kd$.
\end{lemma}

\begin{proof}
Assume for the sake of contradiction that there exists $A = \lrset{f_1,\cdots,f_{kd+1}} \subseteq \cA$, $S = \lrset{x_1,\cdots,x_{kd+1}} \subseteq \bigcup_p X_p$, and witness $\theta$ such that for every $i,j \in [kd+1]$
\begin{equation}
\begin{aligned}
    f_{i}(x_j) &\geq \theta+\epsilon \qquad &\text{for all }i\leq j\\
    f_{i}(x_j) &\leq \theta \qquad &\text{for all }i> j
\end{aligned}
\end{equation}

By the pigeon hole principle, there must exist a $p \in [k]$ such that $|S \cap X_p| \geq d+1$.  Therefore, for every $i,j$ with $x_i,x_j \in S \cap X_p$

\begin{equation}
\begin{aligned}
    f_{p,i}(x_j) &\geq \theta+\epsilon \qquad &\text{for all }i\leq j\\
    f_{p,i}(x_j) &\leq \theta \qquad &\text{for all }i> j
\end{aligned}
\end{equation}

\noindent contradicting our assumption that $\fatr(\cF_p,\epsilon) \leq d$, as desired.
\end{proof}

\begin{lemma}\label{lem:xoring}
    Let $\cF$ be a 0-1 valued function class. Define the  $\mathrm{XOR}$ of 2 functions as follows: $(f\oplus g)(x) = \hI(f(x)\ne g(x))$. We also define the $\mathrm{XOR}$ of a function class $\cF$ and a function $g$ as follows: $\cF \oplus g = \{f \oplus g : f \in \cF\}$. Then we have that:
    \begin{equation*}
        \tr(\cF \oplus g) \leq 2\tr(\cF) + 1
    \end{equation*}
\end{lemma}

\begin{proof}
    Consider the function class in which all the entries from $\cF$ are flipped and name if $\bar{\cF}$. Now let us concatenate the two function classes $\cF$ and $\bar{\cF}$ vertically, as the functions are the rows of the function class. Name this new function class $\cF_{conc}$. Note that the threshold dimension of $\cF_{conc}$ is going to be at most the sum of the threshold dimensions of $\cF$ and $\bar{\cF}$, i.e.
    \begin{equation*}
        \tr(\cF_{conc}) \leq \tr(\cF) + \tr(\bar{\cF})
    \end{equation*}
    Moreover, one can notice that if $\cF$ has threshold dimension $\tr(\cF)$, then the threshold dimension of $\bar{\cF}$ has to be at least $\tr(\cF)-1$. Similarly, the threshold dimension of $\cF$ has to beat least $\tr(\bar{\cF}) - 1$. So we have $\tr(\cF) \geq \tr(\bar{\cF}) - 1$. Thus we have:
    \begin{equation*}
        \tr(\cF_{conc}) \leq \tr(\cF) + \tr(\bar{\cF}) \leq 2\tr(\cF) + 1
    \end{equation*}
    Finally we will prove that $\tr(\cF \oplus g) \leq \tr(\cF_{conc})$, which combined with the above result will give us the required bound. \\
    Recall that the rows of a function class are the functions. Let us take a look at what happens to the column $j$ of $\cF$ when we $\mathrm{XOR}$ it with the entry of the function $g$, $g_j$, at that column $j$. If $g_j = 0$ then the column is not swapped,  otherwise if $g_j = 1$ the column is swapped. Therefore each column of $\cF \oplus g$ is either a column of $\cF$ or a column of $\bar{\cF}$. That means that any threshold in $\cF \oplus g$ will be contained in $\cF_{conc}$, thus concluding our proof.

\end{proof}

\section{Various inequalities}\label{app:inequalities}

We first prove the following auxiliary lemma:
\begin{lemma}\label{lem:unique-threshold}
    Let $\theta_1,\dots,\theta_d$ be numbers in $[0,1-\epsilon]$ for some $\epsilon>0$. Then, there exists some $\theta\in [0,1]$ such that 
    \[
    \l|\l\{ i \in [d] \colon \theta_i \in [\theta-\epsilon/2,\theta]\r\}\r|
    \ge \frac{\epsilon d}{2}
    \]
\end{lemma}
\begin{proof}
    For each $i \in [d]$, if $\theta$ is drawn uniformly at random from $[0,1]$, then with probability $\epsilon/2$ it holds that $\theta_i \in [\theta-\epsilon/2,\theta]$. Consequently, taking expectation over $\theta$, the expected number of elements $i$ such that $\theta_i \in [\theta-\epsilon/2,\theta]$ is $\epsilon d/2$. There exists some $\theta \in [0,1]$ that realizes this expectation, namely, that there are at least $\epsilon d/2$ elements $i$ such that  $\theta_i \in [\theta-\epsilon/2,\theta]$, as required. 
\end{proof}

Using Lemma~\ref{lem:unique-threshold}, we proceed to providing a sketch of Lemma~\ref{lem:ineq-comb}:

\begin{proof}[Proof sketch of Lemma~\ref{lem:ineq-comb}]
    let $d = \fat(\cF,\epsilon)$. By definition of the fat-shattering dimension, there exists a set $\{x_1,\dots,x_d\}$ and witnesses $(\theta_1,\dots,\theta_d)$ that satisfy Eq.~\eqref{eq:fat-shattered}.
    To prove $\fat(\cF,\epsilon)\le \sfat(\cF,\epsilon)$, notice that we can construct a complete binary tree of depth $d$, such that all internal nodes of depth $i$ are labelled by $x_{i+1}$, for $i=0,\dots,d-1$. Further, Eq.~\eqref{eq:fat-shattered} imply that one could label the leaves with appropriate functions $f \in \cF$ such that Eq.~\eqref{eq:sfsd} holds, which implies that $\sfat(\cF,\epsilon) \ge d$ as required.

    For the inequality $\fat(\cF,\epsilon) \le \frac{\fatr(\cF,\epsilon/2}{2\epsilon}$, we use the same notation of $\{x_1,\dots,x_d\}$ and witnesses $(\theta_1,\dots,\theta_d)$ that satisfy Eq.~\eqref{eq:fat-shattered}, where $d = \fat(\cF,\epsilon)$. By Lemma~\ref{lem:unique-threshold} there is some $\theta \in [0,1]$ such that there exist at least $\epsilon d/2$ elements $i \in [d]$ such that $\theta_i \in [\theta,\theta+\epsilon/2]$. Let $i_1,\dots,i_m$ denote the set of indices of these $\theta_i$ variables, where $m \ge \epsilon d/2$. By Eq.~\eqref{eq:fat-shattered}, for each $j \in [m]$ there exists $f_j \in \cF$ such that $f(x_{i_\ell}) \ge \theta_{i_j} + \epsilon \ge \theta + \epsilon$ for all $\ell \in \{j,j+1,\dots,m\}$ and such that $f(x_{i_\ell}) \le \theta_{i_j} \le \theta$ for all $\ell \in [j-1]$. This, by definition, implies that $\fatr(\cF,\epsilon/2) \ge m \ge \epsilon d/2 = \epsilon \fat(\cF,\epsilon)/2$, as required.

    For the inequality $\fatr(\cF,\epsilon) \le 2^{\sfat(\cF,\epsilon)+1}$, notice that it is equivalent to $\sfat(\cF,\epsilon) \ge \log\fatr(\cF,\epsilon) - 1$. It is sufficient to prove that $\sfat(\cF,\epsilon) \ge \lfloor\log\fatr(\cF,\epsilon)\rfloor$. To obtain that, denote $m = 2^{\lfloor\fatr(\cF,\epsilon)\rfloor}$ and notice that given functions $f_1,\dots,f_m \in \cF$ elements $x_1,\dots,x_m \in \cX$ and $\theta\in [0,1]$ that satisfy Eq.~\eqref{eq:fatrd}, one could construct a tree whose internal nodes are labelled by $x_2,\dots,x_m$ and whose leaves are labelled by $f_1,\dots,f_m$, that satisfies Eq.~\eqref{eq:sfsd}.

    The left part of Eq.~\eqref{eq:sfat-tr} was proved by \cite[Lemma~8.4]{daskalakis2022fast}.
\end{proof}

\end{document}

%% file: header.tex
\usepackage{amsthm} 
\usepackage{mathtools,thmtools}
\usepackage{amsfonts,amsmath,amssymb}
\usepackage[capitalise]{cleveref}
\usepackage{times}
\usepackage{algorithm}
\usepackage{algpseudocode}
\usepackage{thm-restate}
\usepackage{tcolorbox}
\usepackage{lipsum}
\usepackage{enumitem}
\usepackage{bm}
\usepackage{xspace}
\usepackage{nicefrac}
\usepackage{dsfont}
\usepackage{float}

\usepackage{bbm}
\usepackage{mathtools,thmtools}

\declaretheoremstyle[
	    spaceabove=\topsep, 
	    spacebelow=\topsep, 
	    headfont=\normalfont\bfseries,
	    bodyfont=\normalfont\itshape,
	    notefont=\normalfont\bfseries,
	    notebraces={(}{)},
	    postheadspace=0.5em, 
	    headpunct={},
	    postfoothook=\noindent\ignorespaces
    ]{theorem}
\declaretheorem[style=theorem,numberwithin=section]{theorem}

\declaretheoremstyle[
	    spaceabove=\topsep, 
	    spacebelow=\topsep, 
	    headfont=\normalfont\bfseries,
	    bodyfont=\normalfont,
	    notefont=\normalfont\bfseries,
	    notebraces={(}{)},
	    postheadspace=0.5em, 
	    headpunct={},
	    postfoothook=\noindent\ignorespaces
    ]{definition}

\declaretheoremstyle[
        spaceabove=\topsep, 
        spacebelow=\topsep, 
        headfont=\normalfont\bfseries,
        bodyfont=\normalfont,
        notefont=\normalfont\bfseries,
        notebraces={}{},
        postheadspace=0.5em, 
        qed=$\blacksquare$, 
        headpunct={},
        postfoothook=\noindent\ignorespaces
    ]{proofstyle}
\declaretheorem[style=proofstyle,numbered=no,name=Proof]{proof}

\declaretheorem[style=theorem,sibling=theorem,name=Lemma]{lemma}
\declaretheorem[style=theorem,sibling=theorem,name=Corollary]{corollary}

\declaretheorem[style=theorem,sibling=theorem,name=Proposition]{proposition}

\declaretheorem[style=theorem,sibling=theorem,name=Goal]{goal}

\declaretheorem[style=theorem,numbered=no,name=Theorem]{theorem*}
\declaretheorem[style=theorem,numbered=no,name=Lemma]{lemma*}
\declaretheorem[style=theorem,numbered=no,name=Corollary]{corollary*}
\declaretheorem[style=theorem,numbered=no,name=Proposition]{proposition*}
\declaretheorem[style=theorem,numbered=no,name=Claim]{claim*}
\declaretheorem[style=theorem,numbered=no,name=Fact]{fact*}
\declaretheorem[style=theorem,numbered=no,name=Observation]{observation*}
\declaretheorem[style=theorem,numbered=no,name=Conjecture]{conjecture*}

\declaretheorem[style=definition,sibling=theorem,name=Definition]{definition}

\declaretheorem[style=definition,numbered=no,name=Definition]{definition*}
\declaretheorem[style=definition,numbered=no,name=Remark]{remark*}
\declaretheorem[style=definition,numbered=no,name=Example]{example*}
\declaretheorem[style=definition,numbered=no,name=Question]{question*}

\DeclareMathOperator{\E}{\mathbb{E}}

\DeclareMathOperator{\cc}{c}
\renewcommand{\c}{\cc}
\DeclareMathOperator{\fat}{fat}
\DeclareMathOperator{\tr}{tr}
\DeclareMathOperator{\sfat}{sfat}
\DeclareMathOperator{\fatr}{fattr}
\DeclareMathOperator{\Lit}{Lit}
\newcommand{\regret}{\mathcal{R}}
\DeclareMathOperator{\VC}{VC}

\newcommand{\actions}{\cA}
\DeclareMathOperator{\Val}{Val}
\DeclareMathOperator{\conv}{conv}
\DeclareMathOperator{\dconv}{dconv}
\DeclareMathOperator{\convtwo}{conv2}
\DeclareMathOperator{\sRad}{sRad}
\DeclareMathOperator{\Loss}{Loss}
\DeclareMathOperator{\Maj}{Maj}
\DeclareMathOperator{\unif}{unif}

\newcommand{\lr}[1]{\mathopen{}\left(#1\right)}
\newcommand{\Lr}[1]{\mathopen{}\big(#1\big)}

\renewcommand{\l}{\left}
\renewcommand{\r}{\right}

\newcommand{\lrbra}[1]{\mathopen{}\left[#1\right]}

\newcommand{\lrset}[1]{\mathopen{}\left\{#1\right\}}
\newcommand{\Lrset}[1]{\mathopen{}\big\{#1\big\}}

\newcommand{\lrabs}[1]{\mathopen{}\left|#1\right|}

\newcommand{\cA}{\mathcal{A}}
\newcommand{\cB}{\mathcal{B}}

\newcommand{\cF}{\mathcal{F}}
\newcommand{\cG}{\mathcal{G}}
\newcommand{\cH}{\mathcal{H}}
\newcommand{\cI}{\mathcal{I}}

\newcommand{\cL}{\mathcal{L}}
\newcommand{\cM}{\mathcal{M}}
\newcommand{\cN}{\mathcal{N}}
\newcommand{\cO}{\mathcal{O}}

\newcommand{\cS}{\mathcal{S}}

\newcommand{\cX}{\mathcal{X}}
\newcommand{\cY}{\mathcal{Y}}

\newcommand{\hR}{\mathbb{R}}
\newcommand{\hE}{\mathbb{E}}

\newcommand{\hI}{\mathbb{I}}

\newcommand{\mat}[1]{\mathbf{#1}}

\newcommand{\CCE}{\mathbf{CCE}}
\newcommand{\MGS}[1]{\mat{M}^G_{#1}}
\newcommand{\efsd}{$\epsilon$-fat-shattering dimension}
\newcommand{\esfsd}{$\epsilon$-sequential-fat-shattering dimension}
\newcommand{\eftd}{$\epsilon$-fat-threshold dimension}
\newcommand{\bfc}{0-1 function class}
\newcommand{\rfc}{real-valued function class}

